\documentclass[twoside]{article}
\usepackage[preprint]{aistats2027}

\usepackage{amsmath, amssymb, amsthm}
\usepackage{bm, bbm}
\usepackage{enumitem} 
\usepackage{mathtools}
\usepackage{graphicx}
\usepackage[ruled]{algorithm2e}
\usepackage{algorithmic}
\usepackage{accents}
\usepackage{xcolor}
\usepackage{color}
\usepackage{booktabs}
\usepackage[colorlinks=true, allcolors=blue]{hyperref}
\usepackage[textwidth=2.0cm, textsize=scriptsize]{todonotes}
\usepackage{comment}
\usepackage{subcaption}
\usepackage[nice]{nicefrac}
\usepackage{preamble}
\usepackage[round]{natbib}

\usepackage{accents}

\begin{document}

\twocolumn[

\aistatstitle{Hidden Disparities in Ranking:\\ Category-Aware Fairness via Distributional Stability}

\aistatsauthor{ Gayane Taturyan \quad Charlotte Laclau \quad  Stephan Clémençon }

\aistatsaddress{ LTCI, Télécom Paris\\Institut Polytechnique de Paris, France} ]

\begin{abstract}
We study the problem of fair ranking in recommender systems. We provide a
formal probabilistic framework in which ranking policies are distributions
over permutations, and the utility of a user is the exposure given to the
items they interact with, relative to that of an ideal ranking. We consider utility parity across user groups, as well as its category-aware refinement, conditional utility parity, which requires parity within each category of items. Following a recent view of fairness
as distributional stability, we show that both criteria are equivalent to
the stability of the utility under changes in the composition of user
groups. Optimizing over mixtures of ranking policies leads to a second-order
cone program, for which we provide generalization bounds. A numerical study
validates our theoretical findings.
\end{abstract}

\section{INTRODUCTION}
\label{sec:intro}

Recommender Systems (RS), often trained as learning-to-rank models, play an increasing role in access to opportunities. On job platforms, the offers displayed to a candidate are ordered according to their predicted relevance to the candidate's profile. When trained on historical interaction data, these models can reproduce disparities related to sensitive attributes such as gender. For example, women may see fewer high-paying or high-responsibility positions than men \citep{Basileal2021}. Since candidates primarily interact with the highest-ranked offers, what each user sees at the top of their ranking determines the quality of service they receive, and can ultimately affect their career path. 

A natural answer is to require comparable utility across user groups, where the utility of a ranking measures how much exposure it gives to the offers the user is interested in. We refer to this criterion as utility parity. Existing works on user-side fairness in recommendation compare a single utility, aggregated over all offers \citep{li2021-user-fairness, wang2021, yao2017}. Such aggregate parity can be misleading: a ranking model may satisfy it while favouring one group on some categories of offers, such as administrative positions, and the other group on others, such as managerial positions, so that the gaps cancel out in aggregate. A similar phenomenon is known in classification, where parity on a population can hide disparities within its subgroups \citep{pmlr-v80-kearns18a}. Recent work has pointed out this limitation and proposed item category-aware criteria \citep{kheya2025,wang2024}. 

Parity, whether aggregate or per category, also has a dynamic reading. On a job platform, the share of women among job seekers may change over time, for instance when the platform opens to new sectors or regions, and as long as a gap remains, the average quality of the recommendations rises or falls with this balance. We build on this reading, drawing inspiration from \citet{stable2026}, who view fairness in classification as the stability of a prediction rate under perturbations of group proportions, and learn fair classifiers as mixtures of reweighted predictors. Extending this view to ranking is far from straightforward, due to the \textit{global} nature of this learning task. For each observation, a classifier predicts a label, which can be modified without changing the other predicted labels. In contrast, a ranking policy organizes the entire catalogue: each item is assigned a specific position in the list, so that its exposure depends on the position of all other items. And changing the rank of one item can affect the rank of (possibly many) other items. We therefore represent ranking policies through the exposure they allocate. The utility of a user then depends linearly on this exposure, so that average utility plays the role of the prediction rate, and the utility of a mixture of policies is exactly the mixture of their utilities. The stability view thus carries over to ranking, and the resulting problem is solved exactly over mixtures of policies, whereas in classification the risk of a mixture is handled through a surrogate.


{The framework has several benefits. First,} since the utility of a ranking is additive over the offers it contains, it decomposes according to their category, and category-aware constraints arise naturally in this formulation, leading to what we call conditional utility parity. 
Second, utility parity is equivalent to the stability of utility under changes in the proportion of each user group, so that the guarantees hold uniformly over a range of group compositions.
{Third,} the constraints can include a 
{target} of utility level per category, which prevents parity from being reached by lowering the utility that both groups derive from a category. 

\textbf{Contributions} First, we provide a formal probabilistic framework for recommendation setup (Section~\ref{sec:prelim}). Second, we formalize utility parity and
conditional utility parity as stability properties, which yields a
constraint controlling both the group gap and the utility level
(Section~\ref{sec:fairness}).
Third, considering mixtures of policies, we write the empirical problem as a second-order cone program (Section~\ref{sec:methodo}) and discuss its feasibility. Fourth, we provide
finite-sample guarantees on fairness, deployment stability and utility,
with explicit exposure rates for Plackett--Luce policies (Section~\ref{sec:guarantees}). Finally,
we conduct numerical experiments that confirm our theory in practice (Section~\ref{sec:experiments}).

\textbf{Related Works} There is a large amount of literature on fair ranking, most of which focus on item-side fairness through allocated exposure. A seminal work of \citet{singh} formulates fairness of exposure as a linear program over doubly stochastic matrices, from which stochastic rankings are sampled
through a Birkhoff--von Neumann decomposition. \citet{wang2021} build on this approach to enforce user fairness, item fairness and diversity together: they optimize both notions of fairness
as a convex program, and derive the ranking policy through a new
decomposition algorithm. Exposure-based fairness has also been integrated in the training of ranking policies, by
backpropagating through constrained optimization problems~\citep{kotary22,dinh2024}, or through efficient gradient
estimators for Plackett--Luce policies~\citep{Oosterhuis2021}.

Closest to our work are the category-aware criteria of~\citet{kheya2025}
and~\citet{wang2024}. \citet{kheya2025} compare the share of each category
in the top-$K$ lists of men and women and use the resulting gap as a regularizer during training. \citet{wang2024} require a recall-based metric of relevant items to be similar across all pairs of user and item groups, and reweight the training losses with a variant of group distributionally robust optimization. Our criteria differ from both. Unlike the former, our utility only
rewards exposure to relevant offers. Unlike the latter, we compare
user groups within each category, without asking categories to be treated
equally. 

Overall, despite the rich literature, most fair ranking methods are algorithmic and empirical, and do not rely on a probabilistic framework that would allow for statistical guarantees. The existing guarantees are mostly of a different nature. For example, fair rank aggregation algorithms come with approximation guarantees which bound how far their solution can be from the optimum~\citep{NEURIPS2022_974309ef,chakraborty2026fairness, chakraborty2025improved}. Statistical guarantees have been provided in other setups, for instance in learning to rank, where a stochastic ranker built on a pretrained model is calibrated post hoc to guarantee a target level of utility,
while item-exposure fairness is improved but not guaranteed
\citep{guo2024inferencetimestochasticrankingrisk}, or in
recommendation, where an online bandit algorithm decides with high probability whether some user would prefer another user's
recommendations, without degrading the audited system's performance by more than a fixed fraction \citep{do2022}. In contrast, we provide finite-sample guarantees on both the user-side
category-aware fairness and the utility of a learned ranking policy, with
fairness gaps and utility levels controlled uniformly over a range of group compositions. As a result, we observe a clear trade-off curve between utility and fairness, which methods without
guarantees cannot ensure.


\textbf{Notation} Random variables are written in uppercase letters (\eg $G$), their realizations in corresponding lowercase letters (\eg $g$). Vectors are set in boldface (\eg $\bX$ and $\bx$). For an integer \(L>0\), we define \([L]\eqdef \{1,\cdots,L\}\) and write $\Delta_L$ for the probability simplex in $\bbR^L$. We denote by $\fS_L$ the set of all permutations of $[L]$.

\section{BACKGROUND AND PRELIMINARIES}
\label{sec:prelim}
In this section, we provide a formal probabilistic framework for the recommendation setup. 
\subsection{Users, items, relevance}
We consider a recommendation setting where the system ranks a fixed collection of items for a population of users. We denote the fixed finite \emph{item catalogue} by $\cI\eqdef[m]$. In addition, we assume that the item catalogue is partitioned into $A$ disjoint groups, $\cI=\bigcup_{a=1}^{A}\cI_a$, where each $\cI_a$ represents a collection of items sharing a common attribute of interest (\eg category). A \emph{user} is modeled as a random triplet $\bZ = (\bX, G, \bY)$, taking values in $\Z \eqdef \X \times \G \times \Y$, and defined on a probability space $(\Omega,\cF,\Prob)$ with joint law $\cD\in\cM_1(\Z)$. Here, $\bX\in\cX$ is the user--query feature vector, that may contain a static profile of user characteristics, as well as session-level information (\eg search terms, active filters, and other contextual features). A binary \emph{sensitive attribute} is denoted by $G\in\{0,1\}$ and represents a protected characteristic of a user, such as race or gender, with a marginal $p_g\eqdef\Prob(G=g)$. Finally, $\bY=(Y_i)_{i\in[m]}\in\{0,1\}^m$ denotes the relevance vector of the user, where $Y_i = 1$ if the user interacted with item $i$ (\eg clicked, applied), and $Y_i = 0$
otherwise.



\subsection{Ranking policies, exposure, utility}
\begin{definition}[Ranking policy] 
\label{def:policy}
A \emph{ranking policy} is a Markov kernel
$P:\X\to\cM_1(\fS_m)$, $\bx\mapsto P(\,\cdot\mid\bx)$. Given $\bX=\bx$ the system draws a ranking
$\pi\sim P(\,\cdot\mid\bx)$, and $\pi(j)$ is the item placed at position $j\in[m]$. 
\end{definition}
We operate in an \emph{unaware} setup, where the policy does not have access to sensitive attribute $G$ at prediction time. The policy can also be summarised by its position marginals.
\begin{definition}[Marginal rank matrix]
    \label{def:policy-matrix} 
    For a policy $P$ and $\bx\in\cX$, the \emph{marginal rank matrix} is $\bM^{P}(\bx)\in[0,1]^{m\times m}$ with
\(
  \bM^{P}_{ij}(\bx)\eqdef \Prob_{\pi\sim P(\cdot\mid\bx)}\!\big(\pi(j)=i\big) \enspace.
\)
\end{definition}
Each row and column of $\bM^P$ sums to one, making $\bM^P$ doubly stochastic, and thus a point of the Birkhoff polytope $\fB_m$. This formulation accommodates both deterministic and stochastic ranking policies.


\begin{example}[Deterministic sorting policy] A scoring function $s_\theta : \X \times \cI \to \bbR$ induces a deterministic sorting policy $P_\theta (\pi\mid\bx) = \ind{\pi=\pi_\theta(\bx)}\enspace,$ where $\pi_\theta(\bx) = \argsort_{i\in[m]} s_\theta (\bx,i) \enspace$ ranks the items by decreasing score.
\end{example}

\begin{example}[Plackett--Luce policy~\citep{Plackett1975}] A scoring function $s_\theta:\X\times\cI\to\bbR$ induces the Plackett-Luce policy 
\(
P_\theta(\pi\mid\bx)=\prod_{j=1}^m\frac{\exp(s_\theta(\bx,\pi(j))}{\sum_{i=j}^m\exp(s_\theta(\bx,\pi(i))} \enspace,
\) 
from which rankings are sampled by successive Gumbel-max draws. 
\end{example}

Considering that higher-ranked items receive more attention, let $\bv=(v_1,\dots,v_m)$ be a \emph{position-discount vector} with $v_1\geq\cdots\geq v_m\geq 0$. Assuming $\bv\neq\mathbf{0}$, without loss of generality, we normalize so that $\norm{\bv}_1=1$ 
. The position-discount vector allows us to quantify the exposure that each item receives under a given policy.

\begin{definition}[Expected exposure] The expected exposure of item $i\in\cI$ under policy $P$ for a user with features $\bx$ is
\(
\rho_i(P;\bx) \eqdef \Exp_{\pi\sim P(\cdot\mid\bx)}\!\Big[\sum_{j=1}^m v_j\,\ind{\pi(j)=i}\Big] \enspace.
\)
\end{definition}
Equivalently, using the marginal rank matrix notation, we have $\rho_i(P;\bx) =\sum_{j=1}^m v_j\,\bM^P_{ij}(\bx) = \big(\bM^P(\bx)\,\bv\big)_i$ and $\brho(P;\bx)\eqdef\bM^P(\bx)\,\bv\enspace.$ The notion of exposure leads to the utility of the ranking policy, which measures the quality of the policy by aligning the allocated exposure with user's relevant items. 
\begin{definition}[Utility]
The \emph{utility} of $P$ for a user $\bz$ is defined as 
\(
U (P; \bz) \eqdef \sum_{i\in[m]} \rho_i (P;\bx) y_i = \scalar{\brho(P;\bx)}{\by}\enspace.
\)
\end{definition}
However, the given notion of the utility is not comparable across users, as users with fewer interactions have lower overall utility.
\begin{definition}[Normalized utility] We define the normalized utility of $P$ for a user $\bz$ as 
\(
\overline{U}(P;\bz)\eqdef\frac{U(P;\bz)}{V(\bz)} \enspace, 
\)
where 
\(
V(\bz)\eqdef\max\limits_{\bM\in\cB}\scalar{\bM\bv}{\by}=\sum_{j=1}^{\norm{\by}_1}v_j\enspace.
\)
\end{definition}
\begin{assumption}
\label{ass:v-bd} It holds that $\|\bY\|_1\geq1$, hence $V(\bZ)\geq\ubar v>0$, almost surely.
\end{assumption}
The expected normalized utility is defined as 
\(
\cU(P) \eqdef \Exp[\overline{U} (P; \bZ)] \enspace,
\)
and the group-conditional normalized utility  as
\(
\cU_g(P) \eqdef \Exp[\overline{U} (P; \bZ) | G=g]\enspace.
\)
\begin{remark}[Connection to NDCG@$K$]\label{rem:ndcg} The normalized discounted cumulative gain at cutoff $K$ (NDCG@$K$) is one of the most popular metrics used for evaluating rankings. For a
ranking $\pi$ and labels $y$, it is defined as
\(
  \mathrm{NDCG@}K(\pi; \by)=
  \frac{\mathrm{DCG@}K(\pi; \by)}{\max_{\pi' \in \mathfrak{S}_m}
  \mathrm{DCG@}K(\pi'; \by)},
\)
where
\(
  \mathrm{DCG@}K(\pi; \by)= \sum_{j=1}^{K}
  \frac{\by_{\pi(j)}}{\log_2(1+j)}
\). Our normalized utility recovers its binary-relevance case with the truncated logarithmic discount $v_j=\frac{c\ind{j \leq K}}{\log_2(1+j)}$, where $c>0$ normalizes
$\|v\|_1=1$. 
It yields 
\(
\oU (P;\bz) = \Exp_{\pi\sim P(\cdot\mid\bx)} [\mathrm{NDCG@}K(\pi; \by)]\enspace.
\)
\end{remark}

\subsection{Fairness as stability}
In a recent work, \cite{stable2026} define fairness in classification as a property of distributional stability. A classifier $h:\X \to \{0,1\}$ is $\varepsilon$--fair \wrt class of perturbation operators $\mathfrak{P}$, iff
\(
 \sup_{(\mathcal{D},\bP)\in \mathcal{M}\times \mathfrak{P}}\Delta(h;\; (\mathcal{D},\bP)) \leq \varepsilon \enspace,
\)
where $\Delta(h;\; (\mathcal{D},\bP))$ is a statistical distance between the distributions of the prediction $h(X)$ under $\mathcal{D}\in \mathcal{M}$ and under $\bP(\mathcal{D})$. The authors show, that the individual and group fairness can both be expressed as stability under specific choices of $\mathfrak{P}$ and $\Delta$. For group fairness, the perturbations change
the composition of the sensitive groups.

\section{UTILITY PARITY AND DISTRIBUTIONAL STABILITY}
\label{sec:fairness}
In this section, we consider recommendation utility based fairness constraints and then express them as stability under distributional changes.
\subsection{Utility-based parities}
First, we introduce a rather natural notion of Utility Parity, that requires the expected recommendation quality to be the same across user groups.
\begin{definition}[Utility Parity]\label{def:up}
A policy $P$ satisfies \emph{Utility Parity} (UP) iff $\cU_0(P)=\cU_1(P)$.
\end{definition}
Next, we introduce Conditional Utility Parity, that requires parity within each category $\cI_a$. For $a\in[A]$, denote the normalized utility within the category by
\[
\oU_a(P,\bz)\eqdef\frac{\sum_{i\in\cI_a}\rho_i(P;\bx)y_i}{V(\bz)} \enspace, 
\]
and the respective expectations by $\cU^a (P) \eqdef \Exp [\oU_a(P,\bZ)]$ and $\cU^a_g (P) \eqdef \Exp [\oU_a(P,\bZ) \mid G=g]$.

\begin{definition}[Conditional Utility Parity]\label{def:cup}
A policy $P$ satisfies \emph{Conditional Utility Parity} (CUP) iff $\cU^a_0(P)=\cU^a_1(P)$ for every $a\in[A]$.
\end{definition}
Notice that, $\oU(P,\bz) = \sum_{a\in[A]}\oU_a(P,\bz)$ and $\cU_g(P) = \sum_{a\in[A]} \cU^a_g (P)$.
Therefore, CUP implies UP, and enforcing just CUP is enough to ensure both. 



\subsection{Group perturbations and utility rates}
We consider a distributional perturbation, where we change the sensitive group proportions, while keeping the conditionals the same. For a parameter $\xi\in[0,1]$, we define the following perturbed distribution,
\[
\cD_\xi \eqdef \xi \cD(\cdot \mid G=1) + (1-\xi)\cD(\cdot \mid G=0) \enspace.
\]
To ensure that both groups are always present in initial and perturbed distributions, we enforce the following natural assumption.
\begin{assumption}[Bounded proportions]\label{ass:groups}
$p_g\in[\ubar p,\,1-\ubar p]$ and $\xi\in[\ubar\xi,\,1-\ubar\xi]$ for some $\ubar \xi,\ubar p\in(0,\tfrac{1}{2})$ and $g\in\{0,1\}$.
\end{assumption}
Under Assumption~\ref{ass:groups}, we have $\cD_\xi\ll \cD$. Thus, computing Radon--Nikodym derivative $w_\xi(g)\eqdef\frac{\d \cD_\xi}{\d \cD}(g)$, we get the explicit importance weights
\begin{equation}\label{eq:weights}
w_\xi(g) = \frac{\xi}{p_1}\ind{g=1} + \frac{1-\xi}{p_0}\ind{g=0} \enspace.
\end{equation}
Moreover, for every integrable $f$, it holds that
\[
\Exp_{\cD_\xi}[f(\bZ)]=\Exp_\cD \big[w_\xi(G)\,f(\bZ)\big] \enspace.
\]

\begin{definition}[Utility rate functionals] For a ranking policy $P$ and a perturbation parameter $\xi\in[0,1]$, we define the following rate functionals
\begin{align*}
    \Phi(P;\xi)&\eqdef\Exp_{\cD_\xi}[\oU(P;\bZ)] = \Exp[w_\xi(G)\oU(P;\bZ)] \text{ and}\\
    \Phi_a(P;\xi)&\eqdef\Exp_{\cD_\xi}[\oU_a(P;\bZ)] \\
    &\,= \Exp[w_\xi(G)\oU_a(P;\bZ)] \text{ for all } a\in[A] \enspace.
\end{align*}
\end{definition}
We notice that $\Phi(\cdot\,;\,\cdot)=\sum_{a\in[A]}\Phi_a(\cdot\,;\,\cdot)$. Moreover, both functionals admit an affine decomposition through the law of total expectation and \eqref{eq:weights}. For $\Phi_{\smsq}\in\{\Phi;\Phi_a\}$ and $\oU_{\smsq}\in\{\oU;\oU_a\}$, it holds that
\begin{align}
    \label{eq:Phi-decomp}
    &\Phi_{\smsq}(P;\xi)=\xi\,\Exp[\oU_{\smsq} \mid G=1]+(1-\xi)\Exp[\oU_{\smsq}  \mid G=0] \nonumber \\
    &= \Exp[\oU_{\smsq}  \mid G=0] + \xi (\Exp[\oU_{\smsq}   \mid G=1] - \Exp[\oU_{\smsq}   \mid G=0]) \enspace.
\end{align}
This allows us to formulate the following proposition.
\begin{proposition}[Parity and stability]\label{prop:par-stab} Let $P$ be a ranking policy. Then

1. $P$ satisfies UP $\Leftrightarrow$ $\xi\mapsto\Phi(P;\xi)$ is constant on [0,1]. Moreover, the unfairness gap is 
\[
|\gamma(P)| \eqdef \sup_{\xi,\xi'\in[0,1]}\big|\Phi(P;\xi)-\Phi(P;\xi')\big| \enspace.
\]

2. $P$ satisfies CUP $\Leftrightarrow$ $\xi\mapsto\Phi_a(P;\xi)$ is constant on [0,1] for all $a\in[A]$. Moreover, the unfairness gap is 
\[
|\gamma_a(P)| \eqdef \sup_{\xi,\xi'\in[0,1]}\big|\Phi_a(P;\xi)-\Phi_a(P;\xi')\big| \enspace, \forall a\in[A]\enspace.
\]
\end{proposition}
Both criteria are Lipschitz statements about the map $\xi\mapsto\Phi_{\smsq}(P;\xi)$. Since this map is affine, its Lipschitz constant is exactly the unfairness gap. Parity, however, is satisfied just as well by equally poor rankings as by equally good ones. Under UP, maximizing the expected utility $\cU(P)$ under the constraint solves this, however for $CUP$ it does not. Since
$\cU(P) = \sum_{a \in [A]} \cU^a(P)$, parity
within a category may be reached at the expense of both groups. Moreover, any residual
gap makes utility vary with the composition of users. We therefore introduce a constraint that keeps utility, within each category for
conditional utility parity, close to a target across group compositions.

\subsection{Stability to perturbations}
Let $\Xi=U[\ubar\xi,1-\ubar\xi]$ be the perturbation parameter distribution. We denote $\mu_\Xi \eqdef \Exp_{\Xi}[\xi]$ and $\sigma^2_\Xi \eqdef \Var_{\Xi}(\xi)$.
\begin{definition}[Perturbation variance and $t$-variance]\label{def:var}
For $t,t_a\in[0,1]$ and rate functionals $\Phi$ and $\Phi_a$, we define perturbation variance functionals as 
\begin{align*}
  &\cV_\Xi(P)\eqdef\Var_{\xi\sim\Xi}\big(\Phi(P;\xi)\big) \quad\text{and} \\  &\cV^a_\Xi(P)\eqdef\Var_{\xi\sim\Xi}\big(\Phi_a(P;\xi)\big) \quad \forall a\in[A]\enspace,
\end{align*}
and perturbation $t$-variance functionals as 
\begin{align*}
  & \cC_\Xi(P;t)\eqdef\Exp_{\xi\sim\Xi}\big[(\Phi(P;\xi)-t)^2\big]\quad\text{and} \\ &\cC^a_\Xi(P;t)\eqdef\Exp_{\xi\sim\Xi}\big[(\Phi_a(P;\xi)-t_a)^2\big] \quad \forall a\in[A] \enspace.
\end{align*}
\end{definition}

Since $\Phi_{\smsq}(P;\cdot)$ is affine, $\cV^{{\smsq}}_\Xi(P)=\gamma_{\smsq}(P)^2\sigma_\Xi^2$, so that
$|\gamma_{\smsq}(P)|\leq\varepsilon\iff \cV^{\smsq}_\Xi(P)\leq\varepsilon^2\sigma_\Xi^2$. Adding and subtracting $\Exp_\Xi[\Phi_{\smsq} (P;\xi)]$
gives the following bias-variance decomposition
\begin{equation}\label{eq:decomp}
  \cC^{{\smsq}}_\Xi(P;t)=\cV^{{\smsq}}_\Xi(P)+\big(\Exp_\Xi[\Phi_{\smsq}(P;\xi)]-t\big)^2 ,
\end{equation}
so that a bound $\cC^{\smsq}_\Xi(P;t)\leq\delta$ controls both the slope $\gamma_{{\smsq}}(P)$ (fairness) and the mean level of
$\Phi_{{\smsq}}(P;\cdot)$ around the target $t_{{\smsq}}$ (uniform deployment). Moreover, it holds that $|\gamma(P)|\leq\sum_{a\in[A]}|\gamma_a(P)|$ and 
\(
(\cC_\Xi(P;\sum_{a\in[A]} t_a))^{\nicefrac{1}{2}}  \leq  \sum_{a\in[A]}(\cC^a_\Xi(P;t_a))^{\nicefrac{1}{2}}\enspace.
\)

\begin{remark}[On normalization]\label{rem:normalization} Interestingly, under CUP, $\cU^a$ is normalized by $V(\bz)$ and is averaged over all users, so a gap
between $\cU^a_0(P)$ and $\cU^a_1(P)$ may partly reflect differences in
how often the two groups are interested in category $a$. A variant closer to Equal Opportunity would restrict each category to
the users interested in it. Writing $\by^a$ for the restriction of
$\by$ to $\cI_a$, let $S_a = \{z : \by^a \neq 0\}$, normalize by the ideal utility of the category
$V^a(z) = \sum_{j \le \|\by^a\|_1} v_j$, and, for $\bz \in S_a$, $\tilde U^a(P,\bz) \eqdef \sum_{i \in I_a} \rho_i(P;\bx)\, y_i / V_a(\bz)$. The matching perturbation mixes $\cD(\cdot \mid S_a, G=g)$, where $\xi_a$ is the share of group 1 among users in $S_a$. $\tilde\Phi_a(P;\xi_a)$ is then affine in $\xi_a$ with slope $\mathbb{E}[\tilde U^a \mid S_a, G=1] - \mathbb{E}[\tilde U^a \mid S_a, G=0]$, and the optimization problem keeps the same form with $\tilde \cU_a$ in place of $\cU_a$. However, these criteria no longer sum to $\cU$, and thus do not imply UP.
\end{remark}

\section{Methodology : mixture policies}
\label{sec:methodo}
We now propose a methodology for maximizing utility under the stability constraints, by optimizing over mixtures of ranking policies.
\subsection{Linearity in policies}
\begin{definition}[Policy mixture]
For policies $P_1,\cdots,P_L$ and $\balpha\in\Delta_L$ we define \textit{policy mixture} as 
\(
  P^{\balpha}(\,\cdot\mid\bx) \eqdef \sum_{\ell=1}^{L}\alpha_\ell\,P_\ell(\,\cdot\mid\bx) \enspace.
\)
\end{definition}
Note that policy mixture is again a policy, since it is a convex combination of probability measures on $\fS_m$. Moreover, the linearity transfers to the following terms.

\begin{lemma}[Policy linearity]\label{lem:pol-lin} For all $\balpha\in\Delta_L$, all $\bx\in\X$ and 
all $\bz\in\Z$, it holds that
\begin{align*}
&\bM^{P^{\balpha}}(\bx) = \sum_{\ell \in [L]} \alpha_\ell \bM^{P_\ell} \enspace,\  \brho(P^{\balpha};\bx)=\sum_{\ell\in[L]}\alpha_\ell\,\brho(P_\ell;\bx) \enspace, \\
&\oU_{\smsq}(P^{\balpha};\bz)=\sum_{\ell\in[L]}\alpha_\ell\,\oU_{\smsq}(P_\ell;\bz) \enspace, \text{ with } \oU_{\smsq}\in\{\oU;\oU_a\} \enspace.
\end{align*}
\end{lemma}

Note, that $\Phi_{\smsq}$ is jointly bi-affine: it is affine in $\xi$ for fixed $P$ and linear in $P$ for fixed
$\xi$.

\subsection{Problem formulation}
To write in a unified manner, let us denote by $\cA$ the set of active constraints. We will consider $\cA=\{0\}$ for UP and $\cA=[A]$ for CUP. With this notation, we have $\gamma_0 \equiv \gamma,\, \cC^0_\Xi \equiv \cC_\Xi,\,t_0\equiv t,\, \delta_0\equiv \delta$. Given a policy dictionary $P_1,\cdots,P_L$, as well as $t_a \in [0,1]$ and $\delta_a > 0$ for all $a\in\cA$, our objective is to solve the following problem
\begin{equation}
    \label{eq:prob:pop}
    \tag{$\cP$}
    \balpha^\star \in \argmax_{\balpha\in\Delta_L} \{\cU(P^\balpha) : \cC^a_\Xi(P^\balpha;t_a) \leq \delta_a\,, \, \forall a \in \cA\} \enspace. 
\end{equation}
Enforcing UP requires one constraint, and enforcing CUP requires one constraint per category (\ie $A$ constraints). 

\textbf{Empirical problem} Let $\cD^{n'} = \{\bZ_1',\cdots,\bZ'_{n'}\}$ be an \iid training sample from $\cD$. Let $\hat P_1,\cdots,\hat P_L$ be policies trained on $\cD^{n'}$, \eg under the importance weights $w_{\xi_\ell}$, and denote by $P_\ell$ the population counterpart of $\hat P_\ell$. Let $\cD^n = \{\bZ_1,\cdots,\bZ_n\}$ be independent \iid sample from $\cD$, and $\Xi^N = \{\xi_1,\cdots,\xi_N\}$ be  perturbation parameters (stratified) sampled from $\Xi$. For simplicity, we assume $p_g$ to be known. For $a\in\cA$, define a matrix $\hat{\bT}^a\in\bbR^{N\times L}$ with elements $\hat{T}_{k,\ell}^a = \hat{\Phi}_{n}^a(\hat P_\ell,\xi_k)$, where $\hat{\Phi}_{n}^a(\hat P_\ell,\xi_k) = \frac{1}{n} \sum_{i\in[n]} w_{\xi_k}(G_i){\oU}_a(\hat P_\ell;\bZ_i)$. The empirical $t$-variance functional is $\hat{\cC}^a_{n,N}(\balpha;t_a) = \frac{1}{N}\sum_{k\in[N]}(\hat{\Phi}_n^a(\hat{P}^\balpha;\xi_k)-t_a)^2 = \frac{1}{N}\|{\hat{\bT}_a\balpha - t_a \mathbf{1}}\|_2^2$. Let us also define a vector $\hat{\bu}\in\bbR^L$, with elements $\hat{u}_\ell = \frac{1}{n}\sum_{i\in[n]}\oU(\hat{P}_\ell,\bZ_i)$.

With this notation, we can write the empirical problem as a second-order cone program (SOCP),
\begin{equation}
\label{eq:prob:emp}
\tag{$\hat{\cP}$}
\hat\balpha\in\argmax_{\balpha\in\Delta_L} \big\{\hat{\bu}^\top\balpha: 
\|\hat{\bT}_a\balpha - t_a \mathbf{1}\|_2 \leq \sqrt{N\delta_a}\enspace,\forall a \in \cA \big\}.
\end{equation}

\subsection{Existence and feasibility} 
\textbf{Existence} The feasible sets of both~\eqref{eq:prob:pop} and~\eqref{eq:prob:emp} are convex and compact, and whenever they are non-empty, an optimizer exists.

\textbf{Feasibility} In classical fairness problems the (strict) feasibility of the constraint is ensured by considering constant functions. When considering mixtures of classifiers, the feasibility can be ensured by adding constant function(s) to the dictionary and choosing appropriate weights. However, in exposure-based fairness problems feasibility isn't always guaranteed. For example, in the work of \cite{singh}, the authors present a specific conditions under which the constraint is attainable for their objective and explain how to handle the issue in practice, however they do not guarantee the feasibility in general. Similarly, feasibility cannot be guaranteed in general for our objective and must instead be characterized. These are conditions that can be checked in practice, and we discuss them in Appendix~\ref{app:feasibility}. From a theoretical perspective, we impose the following simplified assumption.
\begin{assumption}[Strict feasibility]\label{ass:feas} There exists $\balpha^0 \in \Delta_L \, \st \cC^a_\Xi(P^{\balpha^0};t_a) < \delta_a\,, \forall a \in \cA$. 
\end{assumption}

\section{Theoretical guarantees}
\label{sec:guarantees} 
We now study how the empirical solution $\hat P^{\hat\balpha}$ of
\eqref{eq:prob:emp} approaches its population counterpart
$P^{\balpha^\star}$ of \eqref{eq:prob:pop}. By Lemma~\ref{lem:pol-lin}, the rate
functionals and the utility depend on a policy only through its exposure, thus a natural requirement on the dictionary is a rate on exposures.
\begin{assumption}[Policy-exposure rates] \label{ass:pol-exp} For any $\eta_{\delta_{n'}}\in(0,1)$, with probability greater than $1-\eta_{\delta_{n'}}$ it holds that
\(
\max_{\ell\in[L]}\Exp \|{\brho(P_\ell;\bX) - \brho(\hat P_\ell;\bX)}\|_1 \leq \delta_{n'} \enspace.
\)
\end{assumption}
In general we want $\delta_{n'}\to0$, which can be computed for specific policies. For example, the following Lemma states the exposure rates for a Plackett-Luce policies.
\begin{lemma}[Plackett-Luce rates] 
\label{lem:rates:PL}
Let $P_\ell$ and $\hat{P}_\ell$ be Plackett-Luce policies with score functions $s_\ell$ and $\hat{s}_\ell$, for $\ell \in [L]$. For every $\bx\in\X$, it holds that
\[
\|\brho(P_\ell;\bx) - \brho(\hat{P}_\ell;\bx)\|_1 \leq \sqrt{m-1} \| \bs_{\ell}(\bx,\cdot) - \hat{\bs}_\ell(\bx,\cdot) \|_\infty \enspace.
\]
\end{lemma}
Under mild appropriate noise conditions, fast rates for the statistical recovery of optimal ranking/scoring policies have been obtained in \cite{pmlr-v83-clemencon18a}, extending the results for fast empirical ranking aggregation proved in \cite{pmlr-v54-korba17a} 
(see also \cite{pmlr-v108-vogel20a} for one-versus-one approaches).

The error comes from three independent samples, each with its own rate. The dictionary policies are trained on $\cD^{n'}$; by Lemma~\ref{lem:pol-lin}, the exposure error of Assumption~\ref{ass:pol-exp} transfers to the rate functionals of any mixture with an error $\delta'_{n'}\eqdef\nicefrac{\delta_{n'}}{(\ubar{p}\,\ubar{v})}$, where $\nicefrac{1}{\ubar{p}}$ bounds the importance weights and $\nicefrac{1}{\ubar{v}}$ comes from the normalization of the utility. Given the dictionary, the rate functionals and the utility are estimated on $\cD^n$, which contributes $\delta_n\asymp n^{-\frac{1}{2}}$ and $\delta_n(\cU)\asymp n^{-\frac{1}{2}}$, respectively. Finally, the perturbation law $\Xi$ is discretized by the stratified sample $\Xi^N$, and since the rate functionals are affine in $\xi$, this contributes $\delta_N\asymp N^{-\frac{3}{2}}$.

\begin{figure*}[!t]
    \centering
    \includegraphics[width=0.95\linewidth]{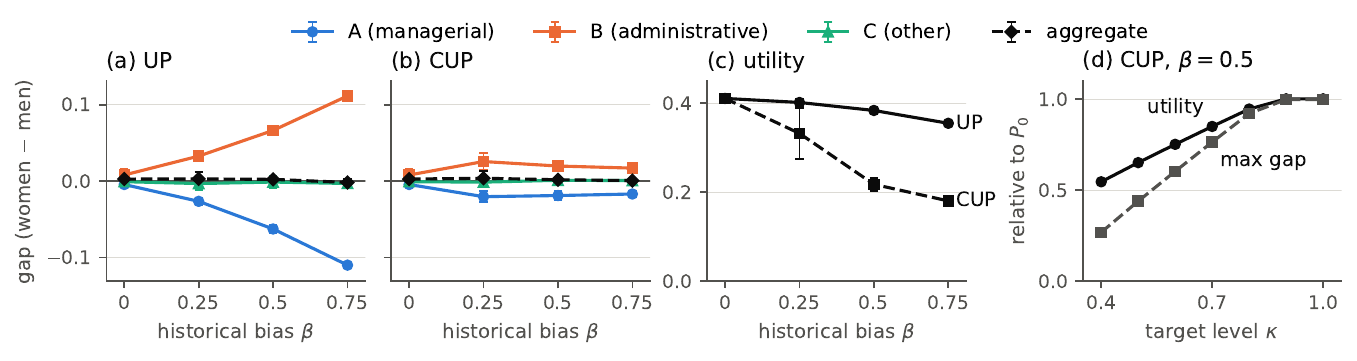}
    \caption{Synthetic data. (a, b) Utility gaps (women $-$ men) under UP and CUP
with free targets. (c) Utility. (d) CUP with targets $t_a = \kappa\,\Phi_a(P_0;\mu_\Xi)$ at $\beta = 0.5$, relative to $P_0$. Means and standard deviations over 3 seeds.}
    \label{fig:toy}
\end{figure*}

\begin{theorem}[Informal]
\label{thm:rates:main}
Suppose Assumptions~\ref{ass:v-bd}--\ref{ass:pol-exp} hold. Then, conditionally
on $\cD^{n'}$ and with ``high probability'' over $(\cD^n,\Xi^N)$, every
$\hat\balpha$ feasible for \eqref{eq:prob:emp} satisfies, for all
$a\in\cA$:

\textit{1. Stability.} $\cC^a_\Xi(\hat P^{\hat\balpha};t_a)\leq(\sqrt{\delta_a}+\delta_n)^2 +\delta_N$.

\textit{2. Unfairness.} $|\gamma_a(\hat P^{\hat\balpha})|\leq\nicefrac{(\sqrt{\delta_a}+\delta_n
      +\sqrt{\delta_N})}{\sigma_\Xi}$.

\textit{3. Uniform deployment.} For all
      $\xi'\in[\ubar{\xi},1-\ubar{\xi}]$, simultaneously, 
      \[
      |\Phi_a(\hat P^{\hat\balpha};\xi')-t_a| \leq (\sqrt{\delta_a}+\delta_n+\sqrt{\delta_N})
        (1+\nicefrac{|\xi'-\mu_\Xi|}{\sigma_\Xi}).
      \]
Moreover, with ``high probability'' over $(\cD^{n'},\cD^n,\Xi^N)$,
\eqref{eq:prob:emp} is feasible and, for some constant $0\leq C<\infty$,
every solution $\hat\balpha$ satisfies

\textit{4. Utility.} $\cU(P^{\balpha^\star})-\cU(\hat P^{\hat\balpha}) \leq  2\delta_n(\cU)+\nicefrac{\delta_{n'}}{\ubar{v}} \\ \text{ }\qquad + C\Big(2\max_{a\in\cA}\sqrt{\delta_a}\,(\delta_n
+\delta'_{n'})+\delta_N\Big) $.
\end{theorem}

All rates are within $\tilde\bigO(n^{{-\frac{1}{2}}}+N^{-\frac{3}{4}}+\delta_{n'})$, and the stability and utility bounds enjoy the faster $N^{-\frac{3}{2}}$ rates. The number of policies $L$ and of categories $|\mathcal{A}|$ only enter
$\delta_n$ and $\delta_n(\mathcal{U})$ through $\log{L}$ and
$\log{|\mathcal{A}|}$, which are absorbed in the $\widetilde{\mathcal{O}}$,
so that enlarging the dictionary or refining the categories does not affect the convergence rates.
Formal statements and proofs are presented in Appendix~\ref{app:guarantees}.

\section{Experiments}
Finally, we illustrate our framework on synthetic and real datasets.
\label{sec:experiments}
\subsection{Synthetic Data}
\label{exp:toy}
\textbf{Setup}
We simulate a simplified job platform with $8000$ users, and a catalogue of $300$ offers divided into three categories: managerial ($A$), administrative ($B$) and other ($C$). The true interests of users do not depend on their group, so an oracle policy has no gap in any category. The policies are however learned from biased historical interactions: with probability $\beta$, an interest of a woman in a managerial offer, or of a man in an administrative offer, was not recorded. The constraints are estimated, and the policies evaluated, on held-out users labelled with their true interests. To illustrate the deployment guarantee, we also evaluate them on new users drawn with a share of women between $0.1$ and $0.9$. The results, deferred to Appendix~\ref{app:deploy}, are in line with Theorem~\ref{thm:rates:main}.  The data generation is detailed in Appendix \ref{app:toy}.

\textbf{Policies and constraints}
The dictionary contains the uniform policy, an unweighted LambdaMART ranker $P_0$, and five LambdaMART rankers trained with the weights $w_\xi$, $\xi \in \{0.1, 0.3, 0.5, 0.7, 0.9\}$. Rankers are trained with LightGBM on
PureSVD embeddings of the history, and the sensitive attribute is never used as a feature. Policies rank deterministically, and the utility is the NDCG@20 (Remark~\ref{rem:ndcg}). We solve~\eqref{eq:prob:emp} with $\Xi$ uniform on $[0.1, 0.9]$. The budgets $\delta_a$ are set so that the tolerated gap in each category is $\lambda$ times the estimation noise of the gap of $P_0$, with $\lambda = 1$ unless stated otherwise. We first leave the targets free, which amounts to constraining the gaps only, and then set each target to a fraction $\kappa$ of the level of $P_0$, $t_a = \kappa\,\Phi_a(P_0;\mu_\Xi)$. More details are available in Appendix~\ref{app:impl}.

\textbf{Results.} Figure~\ref{fig:toy} reports, on the test users, the gaps and the utility of the mixture selected under each constraint. Under UP (a), the aggregate gap remains close to zero for every $\beta$, so the constraint is never active and the selected mixture reduces to the ranker with the highest estimated utility. Meanwhile, women are increasingly under-served on managerial offers and men on administrative ones. UP is therefore blind to disparities that cancel out across categories. 
Under CUP (b), the gaps remain close to zero in every category. We also analyse the price of this correction in (c) and note that with this dictionary, removing the category gaps requires moving towards uniform rankings, which lowers utility for both groups. Finally, we set the targets ($\beta = 0.5$, $\lambda = 5$). For $\kappa \geq 0.9$, the targets forbid lowering the utility of any category, and since this dictionary cannot reduce the gaps otherwise, the selected policy remains a single ranker, $P_0$ on two seeds out of three. Lowering $\kappa$ lets the utility decrease roughly in proportion, while the largest gap shrinks faster. The targets thus make the loss of utility explicit and leave its extent to the practitioner.

\begin{figure*}
    \centering
    \begin{subfigure}[b]{0.68\linewidth}
        \includegraphics[width=\linewidth]{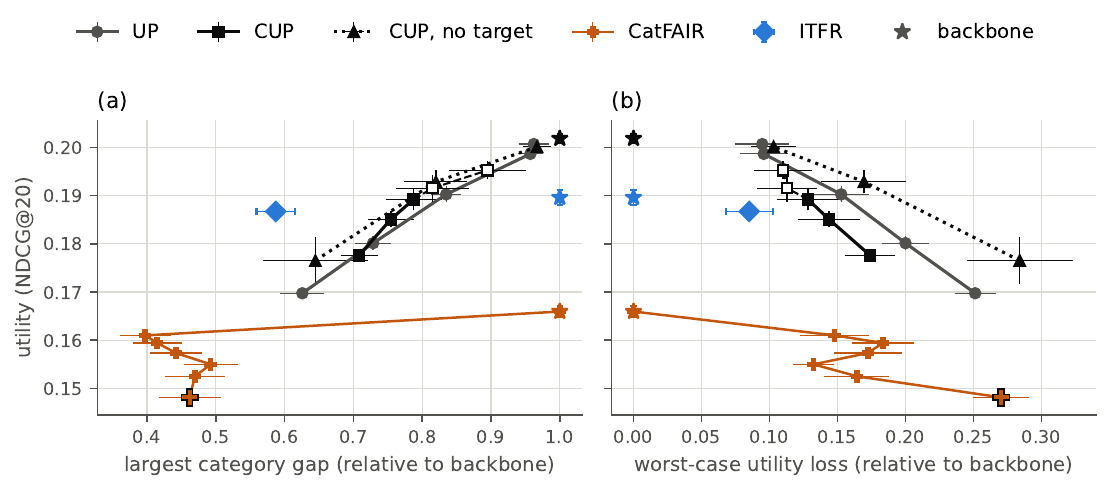}
    \end{subfigure}
    \begin{subfigure}[b]{0.28\linewidth}
        \includegraphics[width=\linewidth]{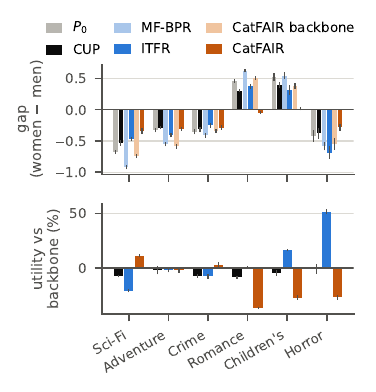}
    \end{subfigure}
    
\caption{MovieLens-1M. Each policy is compared with its backbone $B$. Utility against the largest gap between women and men
over genres (a), and against the worst-case utility loss $1-\min_a \min_\xi \Phi_a(P;\xi)/\min_\xi \Phi_a(B;\xi)$ (b), both relative
to $B$; the best policies lie in the upper left corner. UP and CUP vary the target level $\kappa$ from $1$ (top) to $0.8$ (bottom); hollow squares are values feasible on part of the splits only ($6/10$ for $\kappa=1$, $9/10$ for $\kappa=0.95$). CUP without target varies the budget $\lambda$, and CatFAIR the fairness regularisation parameter. (c) Per genre, for CUP ($\kappa=0.9$) and the other baselines next to their backbones: gap between women and men divided by the utility of the genre under $B$ (top), and change of this utility with respect to $B$ (bottom). Means and standard errors over 10 splits.}
    \label{fig:ml1m}
\end{figure*}

\subsection{MovieLens-1M}
\textbf{Set-up and baselines}
We use MovieLens-1M with gender as the sensitive attribute ($G=1$ for women). Item categories are the six movie genres (Sci-Fi, Adventure, Crime, Romance, Children's, Horror), and we keep the movies that belong to exactly one of them. The interactions of each user are sorted by time: the first 80\% form the history $\bx$ and the remaining 20\% are the relevant items. The candidates of a user are all movies outside their history. Users are split at random into training (50\%), estimation (25\%), on which the mixture weights are fitted, and test (25\%) users, and results are averaged over 10 random splits. Statistics are given in Appendix \ref{app:ml1m}. We compare with ITFR \citep{wang2024} and CatFAIR \citep{kheya2025}, which, like CUP, consider disparities between (user group, item category) pairs: ITFR through a recall-based measure of relevant items, CatFAIR through the share of each category in the recommendation lists, regardless of relevance. Both modify the training of a matrix factorization model with the BPR loss \citep{rendle2009}, see the same interactions as the rankers, and use $G$ during training only. Each is compared with its own backbone trained without the fairness component (MF-BPR for ITFR, $\alpha = 0$ for CatFAIR), as $P_0$ is for our approach. 

The dictionary of Section~\ref{exp:toy} is enriched with several rankers to improve feasibility. In the end, we consider $L = 19$, detailed in Table \ref{tab:ml1m-dict} (Appendix \ref{app:ml1m})). Targets are $t_a = \kappa\,\Phi_a(P_0;\mu_\Xi)$, and $\kappa$ varies from $1$ to $0.8$ to trace the trade-off between utility and gaps, with $\lambda = 2$, the smallest budget of our grid for which the targets with $\kappa \le 0.9$ can be met on every split. CUP without target uses $\lambda \in \{1.5, 2, 3\}$.

On the test users, we report the utility, the largest category gap $\max_a |\gamma_a|$, and the worst-case utility loss $1-\min_a \min_\xi \Phi_a(P;\xi)/\min_\xi \Phi_a(\text{Backbone};\xi)$, which measures how much the less-served population loses in the most affected category.  

\textbf{Utility and fairness.}
We see from Figure~\ref{fig:ml1m} that CUP, ITFR and CatFAIR all reduce the
gaps between women and men across genres with respect to their backbones,
which start from different levels of utility, $P_0$ being the most accurate.
The three approaches differ in how the trade-off is controlled. First, ITFR
gives a single operating point, as none of its hyperparameters sets the
trade-off \citep[App.~A.4]{wang2024}. CatFAIR reaches its smallest gap with
the smallest weight $\alpha$, and larger values only lower utility (a). CUP
instead traces a front: lowering $\kappa$ reduces the largest gap at the cost
of utility, and any point of this front is obtained in milliseconds from the
same dictionary. Second, the targets prevent levelling down (b). Without
targets, CUP reaches the same utility and largest gap, but the less-served
group loses much more utility in the most affected genre. So does CatFAIR,
which ignores relevance, with the weight used by its authors.

Now turning to the question of which movie genres are affected, we see from
Figure~\ref{fig:ml1m}(c) that the three methods reduce the gaps in different
ways. CUP narrows the gap of $P_0$ in every genre and changes the utility of
each genre by less than $10\%$. ITFR moves utility across genres, from $-21\%$
for Sci-Fi to $+51\%$ for Horror, where the gap widens. CatFAIR closes the
gaps in Romance and Children's, but by lowering the utility of women in these
genres, which is halved in Romance.

\textbf{Choosing the trade-off in practice.} For a practitioner, the level
of fairness is a choice that must be made after seeing its cost in utility,
and possibly revised: an explicit front is therefore as important as the
quality of a single operating point. Our guarantees are what make such a front reliable: the budgets and
targets set the gaps and utility levels directly, and
Theorem~\ref{thm:rates:main} ensures that these settings carry over to new
users, whereas ITFR and CatFAIR control fairness only through training
hyperparameters, with no guarantee on the resulting gaps. With CUP, the dictionary is trained once
(about $2$\,min per split), and each point of the front, set by interpretable
targets, is a second-order cone program solved in less than $10$\,ms. With
ITFR and CatFAIR, each point requires a full training (about $38$ and
2\,min on the same CPU), and finding the trade-off means
cross-validating hyperparameters whose effects on fairness are not monotone.

\section{CONCLUSION}
\label{sec:conclusion}
We studied user-side category-aware fair ranking through the lens of distributional stability. Representing ranking policies through exposure, we
showed that utility parity and its category-aware refinement can be viewed as the stability of utility under changes in the composition of the population, and that targets fix the level of service that each category keeps, whatever this composition. Considering mixtures of policies leads to a second-order cone program, for which we provided finite-sample guarantees on stability, fairness, deployment and utility. Since it only combines existing ranking policies, the approach is model-agnostic: it can be applied on top of the most accurate ranker available, and the trade-off between utility and fairness can be set, and revised, after training at negligible cost. Our experiments showed that aggregate parity can hide opposite disparities across categories, which conditional utility parity removes, illustrated the deployment guarantee on new populations, and showed that the targets give an explicit control of the trade-off while limiting levelling down. Our guarantees cover changes in the share of each group; changes in the behaviour of each group are left for future work.

\newpage
\subsection*{AI use statement}
In this work, we used generative AI tools to provide feedback on research and
assist with translation. We also use it for the experiment section, notably to clean and reformat the real-world dataset and to produce figures from the result files obtained.
We have not used generative AI tools to generate synthetic data sets, help develop theoretical models or conceptual frameworks, formulate mathematical claims, provide critical ingredients for proving mathematical claims, assist in the writing of proofs, propose or refine hypotheses, implement methods, support qualitative and thematic data analysis, interpret results.

We have reviewed all AI-assisted work. For the code, we checked the correctness of the pre-processing code by cross-comparison with statistics reported by published papers that rely on the same dataset. For the figure, we manually checked that the numbers displayed are coherent with the results from the files.
We take responsibility for the final content of this work, including text, claims, or artifacts produced with the aid of generative AI.


\bibliography{bibliography}

@InProceedings{pmlr-v83-clemencon18a,
  title = 	 {Ranking Median Regression: Learning to Order through Local Consensus},
  author = 	 {Clémençon, Stephan and Korba, Anna and Sibony, Eric},
  booktitle = 	 {Proceedings of Algorithmic Learning Theory},
  pages = 	 {212--245},
  year = 	 {2018},
  editor = 	 {Janoos, Firdaus and Mohri, Mehryar and Sridharan, Karthik},
  volume = 	 {83},
  series = 	 {Proceedings of Machine Learning Research},
  month = 	 {07--09 Apr},
  publisher =    {PMLR},
  url = 	 {https://proceedings.mlr.press/v83/clemencon18a.html}
}

@InProceedings{pmlr-v54-korba17a,
  title = 	 {{A Learning Theory of Ranking Aggregation}},
  author = 	 {Korba, Anna and Clemencon, Stéphan and Sibony, Eric},
  booktitle = 	 {Proceedings of the 20th International Conference on Artificial Intelligence and Statistics},
  pages = 	 {1001--1010},
  year = 	 {2017},
  editor = 	 {Singh, Aarti and Zhu, Jerry},
  volume = 	 {54},
  series = 	 {Proceedings of Machine Learning Research},
  month = 	 {20--22 Apr},
  publisher =    {PMLR},
  url = 	 {https://proceedings.mlr.press/v54/korba17a.html}
}

@InProceedings{pmlr-v108-vogel20a,
  title = 	 {A Multiclass Classification Approach to Label Ranking},
  author =       {Vogel, Robin and Cl\'emen\c{c}on, St\'ephan},
  booktitle = 	 {Proceedings of the Twenty Third International Conference on Artificial Intelligence and Statistics},
  pages = 	 {1421--1430},
  year = 	 {2020},
  editor = 	 {Chiappa, Silvia and Calandra, Roberto},
  volume = 	 {108},
  series = 	 {Proceedings of Machine Learning Research},
  month = 	 {26--28 Aug},
  publisher =    {PMLR},
  url = 	 {https://proceedings.mlr.press/v108/vogel20a.html}
}

@article{Plackett1975,
  title     = {The Analysis of Permutations},
  author    = {Plackett, Robert L.},
  journal   = {Applied Statistics},
  volume    = {24},
  number    = {2},
  pages     = {193--202},
  year      = {1975},
  doi       = {10.2307/2346567}
}

@inproceedings{Oosterhuis2021,
author = {Oosterhuis, Harrie},
title = {Computationally Efficient Optimization of Plackett-Luce Ranking Models for Relevance and Fairness},
year = {2021},
isbn = {9781450380379},
publisher = {Association for Computing Machinery},
address = {New York, NY, USA},
url = {https://doi.org/10.1145/3404835.3462830},
doi = {10.1145/3404835.3462830},
booktitle = {Proceedings of the 44th International ACM SIGIR Conference on Research and Development in Information Retrieval},
pages = {1023–1032},
numpages = {10},
location = {Virtual Event, Canada},
series = {SIGIR '21}
}

@inproceedings{wang2021,
author = {Wang, Lequn and Joachims, Thorsten},
title = {User Fairness, Item Fairness, and Diversity for Rankings in Two-Sided Markets},
year = {2021},
isbn = {9781450386111},
publisher = {Association for Computing Machinery},
address = {New York, NY, USA},
url = {https://doi.org/10.1145/3471158.3472260},
doi = {10.1145/3471158.3472260},
booktitle = {Proceedings of the 2021 ACM SIGIR International Conference on Theory of Information Retrieval},
pages = {23–41},
numpages = {19},
location = {Virtual Event, Canada},
series = {ICTIR '21}
}

@inproceedings{singh,
author = {Singh, Ashudeep and Joachims, Thorsten},
title = {Fairness of Exposure in Rankings},
year = {2018},
isbn = {9781450355520},
publisher = {Association for Computing Machinery},
address = {New York, NY, USA},
url = {https://doi.org/10.1145/3219819.3220088},
doi = {10.1145/3219819.3220088},
booktitle = {Proceedings of the 24th ACM SIGKDD International Conference on Knowledge Discovery \& Data Mining},
pages = {2219–2228},
numpages = {10},
location = {London, United Kingdom},
series = {KDD '18}
}

@inproceedings{li2021-user-fairness,
author = {Li, Yunqi and Chen, Hanxiong and Fu, Zuohui and Ge, Yingqiang and Zhang, Yongfeng},
title = {User-oriented Fairness in Recommendation},
year = {2021},
isbn = {9781450383127},
publisher = {Association for Computing Machinery},
address = {New York, NY, USA},
url = {https://doi.org/10.1145/3442381.3449866},
doi = {10.1145/3442381.3449866},
booktitle = {Proceedings of the Web Conference 2021},
pages = {624–632},
numpages = {9},
location = {Ljubljana, Slovenia},
series = {WWW '21}
}

@misc{stable2026,
      title={A Comprehensive View of Fairness through Distributional Stability}, 
      author={Gayane Taturyan and Charlotte Laclau and Stephan Clémencon},
      year={2026},
      eprint={2609.37061},
      archivePrefix={arXiv},
      primaryClass={cs.LG},
      url={https://arxiv.org/abs/2609.37061}, 
      note={arXiv preprint arXiv:2609.37061}
}

@inproceedings{dinh2024,
author = {Dinh, My H and Kotary, James and Fioretto, Ferdinando},
title = {Learning Fair Ranking Policies via Differentiable Optimization of Ordered Weighted Averages},
year = {2024},
isbn = {9798400704505},
publisher = {Association for Computing Machinery},
address = {New York, NY, USA},
url = {https://doi.org/10.1145/3630106.3661932},
doi = {10.1145/3630106.3661932},
booktitle = {Proceedings of the 2024 ACM Conference on Fairness, Accountability, and Transparency},
pages = {2508–2517},
numpages = {10},
location = {Rio de Janeiro, Brazil},
series = {FAccT '24}
}

@inproceedings{kotary22,
author = {Kotary, James and Fioretto, Ferdinando and Van Hentenryck, Pascal and Zhu, Ziwei},
title = {End-to-End Learning for Fair Ranking Systems},
year = {2022},
isbn = {9781450390965},
publisher = {Association for Computing Machinery},
address = {New York, NY, USA},
url = {https://doi.org/10.1145/3485447.3512247},
doi = {10.1145/3485447.3512247},
booktitle = {Proceedings of the ACM Web Conference 2022},
pages = {3520–3530},
numpages = {11},
location = {Virtual Event, Lyon, France},
series = {WWW '22}
}

@inproceedings{
chakraborty2026fairness,
title={Fairness in Aggregation: Optimal Top-\$k\$ and Improved Full Ranking},
author={Diptarka Chakraborty and Arya Mazumdar and Barna Saha and Alvin Hong Yao Yan},
booktitle={Forty-third International Conference on Machine Learning},
year={2026},
url={https://openreview.net/forum?id=CeyNBFDCdB}
}

@inproceedings{Basileal2021,
author = {Imana, Basileal and Korolova, Aleksandra and Heidemann, John},
title = {Auditing for Discrimination in Algorithms Delivering Job Ads},
year = {2021},
isbn = {9781450383127},
publisher = {Association for Computing Machinery},
address = {New York, NY, USA},
url = {https://doi.org/10.1145/3442381.3450077},
doi = {10.1145/3442381.3450077},
booktitle = {Proceedings of the Web Conference 2021},
pages = {3767–3778},
numpages = {12},
location = {Ljubljana, Slovenia},
series = {WWW '21}
}

@InProceedings{pmlr-v80-kearns18a,
  title = 	 {Preventing Fairness Gerrymandering: Auditing and Learning for Subgroup Fairness},
  author =       {Kearns, Michael and Neel, Seth and Roth, Aaron and Wu, Zhiwei Steven},
  booktitle = 	 {Proceedings of the 35th International Conference on Machine Learning},
  pages = 	 {2564--2572},
  year = 	 {2018},
  editor = 	 {Dy, Jennifer and Krause, Andreas},
  volume = 	 {80},
  series = 	 {Proceedings of Machine Learning Research},
  month = 	 {10--15 Jul},
  publisher =    {PMLR},
  url = 	 {https://proceedings.mlr.press/v80/kearns18a.html}
}

@inproceedings{yao2017,
author = {Yao, Sirui and Huang, Bert},
title = {Beyond parity: fairness objectives for collaborative filtering},
year = {2017},
isbn = {9781510860964},
publisher = {Curran Associates Inc.},
address = {Red Hook, NY, USA},
booktitle = {Proceedings of the 31st International Conference on Neural Information Processing Systems},
pages = {2925–2934},
numpages = {10},
location = {Long Beach, California, USA},
series = {NIPS'17}
}

@inproceedings{kheya2025,
author = {Kheya, Tahsin Alamgir and Bouadjenek, Mohamed Reda and Aryal, Sunil},
title = {Unmasking Gender Bias in Recommendation Systems and Enhancing Category-Aware Fairness},
year = {2025},
isbn = {9798400712746},
publisher = {Association for Computing Machinery},
address = {New York, NY, USA},
url = {https://doi.org/10.1145/3696410.3714528},
doi = {10.1145/3696410.3714528},
booktitle = {Proceedings of the ACM on Web Conference 2025},
pages = {5127–5138},
numpages = {12},
location = {Sydney NSW, Australia},
series = {WWW '25}
}

@inproceedings{wang2024,
author = {Wang, Yifan and Sun, Peijie and Ma, Weizhi and Zhang, Min and Zhang, Yuan and Jiang, Peng and Ma, Shaoping},
title = {Intersectional Two-sided Fairness in Recommendation},
year = {2024},
isbn = {9798400701719},
publisher = {Association for Computing Machinery},
address = {New York, NY, USA},
url = {https://doi.org/10.1145/3589334.3645518},
doi = {10.1145/3589334.3645518},
booktitle = {Proceedings of the ACM Web Conference 2024},
pages = {3609–3620},
numpages = {12},
location = {Singapore, Singapore},
series = {WWW '24}
}

@inproceedings{NEURIPS2022_974309ef,
 author = {Chakraborty, Diptarka and Das, Syamantak and Khan, Arindam and Subramanian, Aditya},
 booktitle = {Advances in Neural Information Processing Systems},
 doi = {10.52202/068431-1740},
 editor = {S. Koyejo and S. Mohamed and A. Agarwal and D. Belgrave and K. Cho and A. Oh},
 pages = {23965--23978},
 publisher = {Curran Associates, Inc.},
 title = {Fair Rank Aggregation},
 url = {https://proceedings.neurips.cc/paper_files/paper/2022/file/974309ef51ebd89034adc64a57e304f2-Paper-Conference.pdf},
 volume = {35},
 year = {2022}
}

@inproceedings{chakraborty2025improved,
  title     = {Improved Rank Aggregation under Fairness Constraint},
  author    = {Chakraborty, Diptarka and Das, Himika and Dey, Sanjana and Yan, Alvin Hong Yao},
  booktitle = {Proceedings of the Thirty-Fourth International Joint Conference on 
               Artificial Intelligence, {IJCAI-25}},
  publisher = {IJCAI Organization},
  editor    = {J. Kwok},
  pages     = {330--338},
  year      = {2025},
  month     = {8},
  note      = {Main Track},
  doi       = {10.24963/ijcai.2025/38},
  url       = {https://doi.org/10.24963/ijcai.2025/38}
}

@misc{guo2024inferencetimestochasticrankingrisk,
      title={Inference-time Stochastic Ranking with Risk Control}, 
      author={Ruocheng Guo and Jean-François Ton and Yang Liu and Hang Li},
      year={2024},
      eprint={2306.07188},
      archivePrefix={arXiv},
      primaryClass={cs.LG},
      url={https://arxiv.org/abs/2306.07188}, 
      note={arXiv preprint arXiv:2306.07188}
}

@inproceedings{rendle2009,
author = {Rendle, Steffen and Freudenthaler, Christoph and Gantner, Zeno and Schmidt-Thieme, Lars},
title = {BPR: Bayesian personalized ranking from implicit feedback},
year = {2009},
isbn = {9780974903958},
publisher = {AUAI Press},
address = {Arlington, Virginia, USA},
booktitle = {Proceedings of the Twenty-Fifth Conference on Uncertainty in Artificial Intelligence},
pages = {452–461},
numpages = {10},
location = {Montreal, Quebec, Canada},
series = {UAI '09}
}

@inproceedings{do2022,
  title     = {Online Certification of Preference-Based Fairness for Personalized Recommender Systems},
  author    = {Do, Virginie and Corbett-Davies, Sam and Atif, Jamal and Usunier, Nicolas},
  booktitle = {Proceedings of the AAAI Conference on Artificial Intelligence},
  volume    = {36},
  number    = {6},
  pages     = {6532--6540},
  year      = {2022},
  doi       = {10.1609/aaai.v36i6.20606}
}
\section*{Checklist}

\begin{enumerate}

  \item For all models and algorithms presented, check if you include:
  \begin{enumerate}
    \item A clear description of the mathematical setting, assumptions, algorithm, and/or model. [Yes]
    \item An analysis of the properties and complexity (time, space, sample size) of any algorithm. [Yes]
    \item (Optional) Anonymized source code, with specification of all dependencies, including external libraries. [Yes]
  \end{enumerate}

  \item For any theoretical claim, check if you include:
  \begin{enumerate}
    \item Statements of the full set of assumptions of all theoretical results. [Yes]
    \item Complete proofs of all theoretical results. [Yes]
    \item Clear explanations of any assumptions. [Yes]
  \end{enumerate}

  \item For all figures and tables that present empirical results, check if you include:
  \begin{enumerate}
    \item The code, data, and instructions needed to reproduce the main experimental results (either in the supplemental material or as a URL). [Yes]
    \item All the training details (e.g., data splits, hyperparameters, how they were chosen). [Yes]
    \item A clear definition of the specific measure or statistics and error bars (e.g., with respect to the random seed after running experiments multiple times). [Yes]
    \item A description of the computing infrastructure used. (e.g., type of GPUs, internal cluster, or cloud provider). [Yes]
  \end{enumerate}

  \item If you are using existing assets (e.g., code, data, models) or curating/releasing new assets, check if you include:
  \begin{enumerate}
    \item Citations of the creator If your work uses existing assets. [Yes]
    \item The license information of the assets, if applicable. [Not Applicable]
    \item New assets either in the supplemental material or as a URL, if applicable. [Yes]
    \item Information about consent from data providers/curators. [Not Applicable]
    \item Discussion of sensible content if applicable, e.g., personally identifiable information or offensive content. [Not Applicable]
  \end{enumerate}

  \item If you used crowdsourcing or conducted research with human subjects, check if you include:
  \begin{enumerate}
    \item The full text of instructions given to participants and screenshots. [Not Applicable]
    \item Descriptions of potential participant risks, with links to Institutional Review Board (IRB) approvals if applicable. [Not Applicable]
    \item The estimated hourly wage paid to participants and the total amount spent on participant compensation. [Not Applicable]
  \end{enumerate}

\end{enumerate}
\clearpage
\appendix
\thispagestyle{empty}
\onecolumn
\aistatstitle{Supplementary Material}
\section*{Organization}
The Supplementary Material is organized as follows.
\begin{itemize}
\item In Appendix~\ref{app:feasibility}, we discuss the characterization of feasiblity.
\item Appendix~\ref{app:aux} collects small proofs for Section~\ref{sec:fairness}.
\item Appendix~\ref{app:guarantees} provides the proofs for Section~\ref{sec:guarantees} regarding the convergence rates of the constraints and the utility, as well as provides rates for a Plackett-Luce policy.
\item Appendix~\ref{app:experiments} provides additional experimental details.
\end{itemize}
\section{Feasibility}
\label{app:feasibility} It follows from Lemma~\ref{lem:pol-lin} that $\gamma_a(P^\balpha) = \sum_{\ell \in [L]} \alpha_\ell \gamma_a(P_\ell)$ and $\gamma_a(P^\balpha)$ is in convex hull of $\{\gamma_a(P_\ell)\}_\ell$. Therefore, if all of them have the same sign, so will $\gamma_a(P^\balpha)$, and it will hold that $\cC^a_\Xi(P^\balpha;t_a)\geq\sigma^2_\Xi \min_\ell \gamma^2_a(P_\ell)$. On the contrary, if $\exists P_\ell, P_{\ell'} \ \st \gamma_a(P_\ell)<0<\gamma_a(P_{\ell'})$, then zero unfairness is attainable, and picking a weight $\tfrac{|\gamma_a(P_{\ell'})|}{|\gamma_a(P_\ell)|+|\gamma_a(P_{\ell'})|}$ for $P_\ell$ and the complement of it for $P_{\ell'}$ will yield $\cC^a_\Xi(P^\balpha;t_a) \leq (\Exp_\Xi[\Phi_{a}(P^\balpha;\xi)]-t_a)^2$, which is zero for $t_a = \Exp_\Xi[\Phi_{a}(P^\balpha;\xi)]$. 
Whether such pairs exist depends on the dictionary, which should contain,
for each category, policies whose gaps have opposite signs. Such policies
can be obtained by diversifying how the dictionary is trained, for
instance by training under the importance weights $w_\xi$ with $\xi$
close to $\ubar{\xi}$ and $1-\ubar{\xi}$. When disparities
originate in the observed relevance rather than in the group
proportions, this may not suffice, and one may additionally emphasize,
during training, the interactions of a given group with a given category.
Overall, for a fixed $t_a$, one might ensure that $\delta_a \geq \min_{\balpha\in\Delta_L} \cC^a_\Xi(P^\balpha;t_a), \forall a\in\cA$, where the minimization is a quadratic program. This condition is sufficient for UP, and for the CUP it must also hold that $\min_{\balpha\in\Delta_L}\max_{a\in\cA}\tfrac{\cC^a_\Xi(P^\balpha;t_a)}{\delta_a}\leq 1$. In practice, these conditions can be checked.
\section{Proofs for Section~\ref{sec:fairness}}
\label{app:aux}
\paragraph{Proof of Proposition~\ref{prop:par-stab}.}
\begin{proposition*}[Parity and stability]
Let $P$ be a ranking policy. Then
\begin{enumerate}
  \item $P$ satisfies UP if and only if $\xi \mapsto \Phi(P;\xi)$ is
  constant on $[0,1]$;
  \item $P$ satisfies CUP if and only if $\xi \mapsto \Phi^a(P;\xi)$ is
  constant on $[0,1]$ for every $a \in [A]$.
\end{enumerate}
Moreover, for $\Phi^\square \in \{\Phi, \Phi^a\}$, one has the explicit
identity
\[
  \Phi^\square(P;\xi) = \cU^\square_0(P) + \xi\,\gamma^\square(P),
  \qquad
  \gamma^\square(P) \eqdef \cU^\square_1(P) - \cU^\square_0(P),
\]
so that
\[
  \sup_{\xi,\xi'\in[0,1]}\big|\Phi^\square(P;\xi) - \Phi^\square(P;\xi')\big|
  = \big|\gamma^\square(P)\big| .
\]
\end{proposition*}

\begin{proof}
Fix $P$ and $\Phi^\square \in \{\Phi, \Phi^a\}$, with the corresponding
$\cU^\square \in \{\cU, \cU^a\}$, and write
$\cU^\square_g(P) = \Exp[\cU^\square(P;Z) \mid G = g]$ for
$g \in \{0,1\}$. 
For every $\xi \in [0,1]$, it holds that
\[
  \Phi^\square(P;\xi)
  = \Exp_{\cD_\xi}\big[\cU^\square(P;Z)\big]
  = \xi\,\cU^\square_1(P) + (1-\xi)\,\cU^\square_0(P)
  = \cU^\square_0(P) + \xi\,\gamma^\square(P).
\]

$(\Rightarrow)$ First, let us show that parity implies stability. If $P$ satisfies UP,
then $\gamma(P) = 0$, and $\Phi(P;\xi) = \cU_0(P)$ for every
$\xi \in [0,1]$, which does not depend on $\xi$. Similarly, if $P$
satisfies CUP, then $\gamma_a(P) = 0$ for every $a \in [A]$, and each
map $\xi \mapsto \Phi^a(P;\xi) = \cU^a_0(P)$ is constant on $[0,1]$.

$(\Leftarrow)$ Now, suppose that $\xi \mapsto \Phi^\square(P;\xi)$ is constant on
$[0,1]$. Then, for all $\xi, \xi' \in [0,1]$,
\[
  0 = \Phi^\square(P;\xi) - \Phi^\square(P;\xi')
    = (\xi - \xi')\,\gamma^\square(P).
\]
Since we may choose $\xi \neq \xi'$, it follows that
$\gamma^\square(P) = 0$, that is,
$\cU^\square_0(P) = \cU^\square_1(P)$ (UP). If the maps $\xi \mapsto \Phi^a(P;\xi)$ are constant for every
$a \in [A]$, we obtain $\cU^a_0(P) = \cU^a_1(P)$ for every $a \in [A]$ (CUP). 

The equivalences are established. Finally, for all $\xi, \xi' \in [0,1]$, the identity above yields
\[
  \big|\Phi^\square(P;\xi) - \Phi^\square(P;\xi')\big|
  = |\xi - \xi'|\,\big|\gamma^\square(P)\big|
  \leqslant \big|\gamma^\square(P)\big| ,
\]
with equality for $(\xi,\xi') = (1,0)$. Therefore,
\[
  \sup_{\xi,\xi'\in[0,1]}\big|\Phi^\square(P;\xi) - \Phi^\square(P;\xi')\big|
  = \big|\gamma^\square(P)\big|
  = \big|\cU^\square_1(P) - \cU^\square_0(P)\big| ,
\]
and the map $\xi \mapsto \Phi^\square(P;\xi)$ is Lipschitz with constant $|\gamma^\square(P)|$.

The proof is concluded.
\end{proof}

Proof of Lemma~\ref{lem:pol-lin}.
\begin{lemma*}[Policy linearity] For all $\balpha\in\Delta_L$, all $\bx\in\X$ and 
all $\bz\in\Z$, it holds that
\begin{align*}
&\bM^{P^{\balpha}}(\bx) = \sum_{\ell \in [L]} \alpha_\ell \bM^{P_\ell} \enspace,\  \brho(P^{\balpha};\bx)=\sum_{\ell\in[L]}\alpha_\ell\,\brho(P_\ell;\bx) \enspace, \\
&\oU_{\smsq}(P^{\balpha};\bz)=\sum_{\ell\in[L]}\alpha_\ell\,\oU_{\smsq}(P_\ell;\bz) \enspace, \text{ with } \oU_{\smsq}\in\{\oU;\oU_a\} \enspace.
\end{align*}
\end{lemma*}
\begin{proof}
The proof follows from linearity of the functional $\bM^{P}_{ij}(\bx)=\int \ind{\pi(j)=i}\,P(\d\pi\mid\bx)$ of the measure $P(\cdot\mid\bx)$; from linearity of $\bv\mapsto\bM\bv$, of the
inner product, and of $\Exp$.
\end{proof}

\section{Proofs for Section~\ref{sec:guarantees}}
\label{app:guarantees}
First, let us recall some notations and introduce additional ones.
We have $\cI_0\equiv\cI$, $\oU_0\equiv\oU$ and
$\Phi_0\equiv\Phi$, and we treat UP ($\cA=\{0\}$) and CUP ($\cA=[A]$)
simultaneously. We also denote 
\[
  \beta_a(P)\eqdef\Exp\big[\oU_a(P;\bZ)\mid G=0\big]
  \quad\text{and}\quad\gamma_a(P)\eqdef\Exp\big[\oU_a(P;\bZ)\mid G=1\big]
  -\Exp\big[\oU_a(P;\bZ)\mid G=0\big] ,
\]
so that $\Phi_a(P;\xi)=\beta_a(P)+\xi\gamma_a(P)$.

\subsection{Auxiliary lemmas}
In this section we present several auxiliary lemmas that we will use throughout the proofs.
\begin{lemma}[Hoeffding's lemma]
\label{lem:hoeffding}
Let $\bX$ be any real-valued random variable \st $a\leq\bX\leq b$ almost surely. Then, for all $\lambda\in \bbR$,
\[
\Exp [e^{\lambda (X-\Exp[X])}] \leq e^{\frac{\lambda^2 (b-a)^2}{8}} \enspace.
\]
\end{lemma}
\begin{lemma}[Hoeffding's inequality]
\label{lem:hoeffding-ineq} Let $\bX_1,\cdots, \bX_N$ be independent random variables such that $\bX_i\in[a_i,b_i]$ for every $i$. Then, for any $t>0$, we have
\[
\Prob\Big(\sum_{i=1}^N (\bX_i - \Exp[\bX_i]) \geq t\Big) \leq e^{-\frac{2t^2}{\sum_{i=1}^N (b_i - a_i)^2}} \enspace.
\]
\end{lemma}
\begin{lemma}[Pinsker's inequality]\label{lem:pinsker} Let $P$ and $Q$ be two probability distributions on a measurable space $(\X,\Sigma)$. Then it holds that
\[
TV(P,Q) \leq \sqrt{\frac{1}{2}\kl(P\|Q)}\enspace.
\]
\end{lemma}
\begin{lemma}[Stratified sampling]\citep[Lemma 6]{stable2026}
\label{lem:rates:stratif} Let  $\xi_1,\cdots,\xi_N$ be stratified samples from $\Xi$, and let $f:[\ubar{\xi},1-\ubar{\xi}]\to\bbR$ be a $L$--Lipschitz function. For all $\eta\in(0,1)$ with probability greater than $1-\eta$, it holds that
\[
\Big|\frac{1}{N}\sum_{i=1}^N f(\xi_i) - \Exp_\Xi[f(\xi)]\Big| \leq \frac{L(1-2\ubar{\xi})}{N^{\nicefrac{3}{2}}} \sqrt{2\log(\nicefrac{2}{\eta})}\enspace.
\]
\end{lemma}

\subsection{Example: Plackett-Luce exposure rate}
We discuss the specific example of Plackett-Luce policies and derive its exposure rates.

Proof of Lemma~\ref{lem:rates:PL}.
\begin{lemma*}[Plackett-Luce rates] Let $P_\ell$ and $\hat{P}_\ell$ be Plackett-Luce policies with score functions $s_\ell$ and $\hat{s}_\ell$, for $\ell \in [L]$. For every $\bx\in\X$, it holds that
\[
\|\brho(P_\ell;\bx) - \brho(\hat{P}_\ell;\bx)\|_1 \leq \sqrt{m-1} \| s_\ell(\bx, \cdot) - \hat{s}_\ell(\bx, \cdot) \|_\infty \enspace.
\]
\end{lemma*}
\begin{proof}
We have that $\rho_i(P;\bx) = \sum_{\pi}P(\pi) v_{\pi^{-1}(i)}$. By the fact that $\norm{\bv}_1=1$ and by Pinsker's inequality (Lemma~\ref{lem:pinsker}) we have
\[
\|{\brho(P_\ell;\bx) - \brho(\hat P_\ell;\bx)}\|_1 \leq \sum_{\pi} |P(\pi)-\hat P(\pi)|\sum_{i}v_{\pi^{-1}(i)} = 2 \tv (P;\hat P) \leq\sqrt{2\kl(P \| \hat P)} \enspace.
\]
From the definition of Plackett-Luce policies, we have 
\begin{align*}
\frac{P(\pi)}{\hat P(\pi)} = \prod_{j=1}^m \frac{e^{s(\pi(j))} \sum_{i=j}^m e^{\hat s(\pi (i))}}{e^{\hat s(\pi(j))} \sum_{i=j}^m e^{s(\pi (i))}} = \prod_{j=1}^m e^{(s-\hat s)(\pi(j))} \frac{\sum_{i=j}^m e^{\hat s(\pi (i))}}{\sum_{i=j}^m e^{s(\pi (i))}}\enspace.
\end{align*}
Taking the logarithm yields
\begin{align*}
\log \frac{P(\pi)}{\hat P(\pi)} = \sum_{j=1}^m \Big[\underbrace{(s-\hat s)(\pi(j)) + \log \frac{\sum_{i=j}^m e^{\hat s(\pi (i))}}{\sum_{i=j}^m e^{s(\pi (i))}}}_{\eqdef S_j} \Big] \enspace,
\end{align*}
and taking the expectation over $\pi\sim P$ yields
$\kl (P \| \hat P) = \sum_{j=1}^m \Exp_{\pi\sim P}\left[S_j \right]$. Now, the sums $\sum_{i=j}^m$ depend on $\pi$ only through $\pi(1),\cdots,\pi(j-1)$, hence, for any function $f(\cdot)$, we have that 
\[
\Exp_{P}[f(\pi(j)) | \pi(1),\cdots,\pi(j-1)] = \frac{\sum_{i=j}^m (e^{s(\pi(i))} f(\pi(i)))}{\sum_{i=j}^m e^{s(\pi(i))}} \enspace.
\]
Moreover, we have 
\[
\frac{\sum_{i=j}^m e^{\hat s(\pi (i))}}{\sum_{i=j}^m e^{s(\pi (i))}} = \sum_{i=j}^m \frac{ e^{\hat s(\pi (i))}}{\sum_{k=j}^m e^{s(\pi (k))}} = \sum_{i=j}^m \frac{ e^{ s(\pi (i))}}{\sum_{k=j}^m e^{s(\pi (k))}}e^{-(s-\hat s)(\pi(i))} \enspace.
\]
Taking $f=e^{-(s-\hat s)}$, we obtain 
\begin{align*}
\Exp_P [S_j | \pi(1),\cdots,\pi(j-1)] 
&= \Exp_P [(s(\pi(j)) - \hat s(\pi(j)) | \ \cdot \ ] + \log \Exp_P [e^{-s(\pi(j)) - \hat s(\pi(j))} | \ \cdot \ ] \\
&= \log \Exp_P [e^{-\big(s(\pi(j)) - \hat s(\pi(j)) - \Exp [s(\pi(j)) - \hat s(\pi(j)) \ | \ \cdot \ ] \big)} | \ \cdot \ ]  \enspace.   
\end{align*}
Noticing that $\max_{i}(s-\hat s) - \min_i(s-\hat s) \leq 2 \|s-\hat s\|_\infty$ and applying Hoeffding's lemma (Lemma~\ref{lem:hoeffding}), we get
\[
\Exp_P [S_j | \pi(1),\cdots,\pi(j-1)] \leq \log e^{\frac{\|s-\hat s\|^2_\infty}{2}} = \frac{\|s-\hat s\|^2_\infty}{2} \enspace.
\]
Since $S_m = 0$, we obtain
\[
\kl (P \| \hat P) = \sum_{j=1}^{m-1} \Exp_{\pi\sim P}\left[S_j \right] \leq (m-1) \frac{\|s-\hat s\|^2_\infty}{2} \enspace.
\]
The proof is concluded.
\end{proof}

\subsection{Error rates by source}
Before proving the main theorem, we identify the different sources of error and quantify their respective contributions to the overall error. Throughout, the samples $\cD^{n'}$, $\cD^n$ and $\Xi^N$ are mutually independent. The dictionary $\hat P_1,\cdots,\hat P_L$ is a function of
$\cD^{n'}$ only, while $\hat\bT_a$, $\hat\bu$ and $\hat\balpha$ are functions of
$(\cD^{n'},\cD^n,\Xi^N)$. For a policy $P$ depending on the data, $\Phi_a(P;\xi)$
and $\cU(P)$ denote expectations over an independent copy $\bZ\sim\cD$,
conditionally on the data. Since $\hat\balpha$ depends on $(\cD^n,\Xi^N)$, all
the bounds involving it are established on events that do not depend on $\hat\balpha$, and hold uniformly over $\balpha\in\Delta_L$. 

Overall, the convergence proofs follow the structure of~\cite{stable2026}, with necessary modifications specific to our setup.

\subsubsection{Policy estimation rates}
\begin{lemma}\label{lem:rates:policy}
Let Assumptions~\ref{ass:v-bd}, \ref{ass:groups} and~\ref{ass:pol-exp} hold. Then, with
probability greater than $1-\eta_{\delta_{n'}}$ over $\cD^{n'}$, for every
$\balpha\in\Delta_L$, $a\in\cA$ and $\xi\in[\ubar \xi,1 - \ubar \xi]$, it holds that
\[
  \big|\Phi_a(\hat P^{\balpha};\xi)-\Phi_a(P^{\balpha};\xi)\big|\leq
  \frac{\delta_{n'}}{\ubar{p}\,\ubar v}\eqdef\delta'_{n'}
  \quad\text{and}\quad
  \big|\cU(\hat P^{\balpha})-\cU(P^{\balpha})\big|\leq\frac{\delta_{n'}}{\ubar v} .
\]
\end{lemma}
\begin{proof}
For a fixed $\bz=(\bx,g,\by)$, we have
\begin{align*}
  \big|\oU_a(\hat P^{\balpha};\bz)-\oU_a(P^{\balpha};\bz)\big|
  &=\frac{1}{V(\bz)}\Big|\sum_{i\in\cI_a}\big(\rho_i(\hat P^{\balpha};\bx)
  -\rho_i(P^{\balpha};\bx)\big)y_i\Big|\leq\frac{1}{\ubar v}\norm{\brho(\hat P^{\balpha};\bx)-\brho(P^{\balpha};\bx)}_1 \enspace.
\end{align*}
By Lemma~\ref{lem:pol-lin}, the triangle inequality and the fact that
$\alpha_\ell\geq0$, $\sum_\ell\alpha_\ell=1$, we have
\begin{align*}
  \Exp\norm{\brho(\hat P^{\balpha};\bX)-\brho(P^{\balpha};\bX)}_1
  &=\Exp\Big\|\sum_{\ell\in[L]}\alpha_\ell\big(\brho(\hat P_\ell;\bX)
  -\brho(P_\ell;\bX)\big)\Big\|_1 \leq\max_{\ell\in[L]}\Exp\norm{\brho(\hat P_\ell;\bX)-\brho(P_\ell;\bX)}_1 \enspace.
\end{align*}
Finally, from Assumptions~\ref{ass:groups} and~\ref{ass:pol-exp}, we get
\begin{align*}
  \big|\Phi_a(\hat P^{\balpha};\xi)-\Phi_a(P^{\balpha};\xi)\big|
  &=\Big|\Exp\Big[w_\xi(G)\big(\oU_a(\hat P^{\balpha};\bZ)
  -\oU_a(P^{\balpha};\bZ)\big)\Big]\Big| \leq\frac{1}{\ubar{p}\ubar v}
  \Exp\norm{\brho(\hat P^{\balpha};\bX)-\brho(P^{\balpha};\bX)}_1
  \leq\frac{\delta_{n'}}{\ubar{p} \ubar v} \enspace.
\end{align*}
The second inequality follows from the same arguments, with $w_\xi\equiv1$
and $\cI_a$ replaced by $\cI$. 

The proof is concluded.
\end{proof}

\subsubsection{Perturbation error rates}
\begin{lemma}\label{lem:rates:pert}
For all $\eta_{\delta_N}\in(0,1)$, with probability greater than
$1-\eta_{\delta_N}$ over $\Xi^N$, for every $P$, $a\in\cA$ and $t_a\in[0,1]$, it holds that
\[
  \Big|\cC^a_\Xi(P;t_a)-\frac{1}{N}\sum_{k=1}^N\big(\Phi_a(P;\xi_k)-t_a\big)^2
  \Big| \leq  \big(2+2(1-\ubar{\xi})\big)\frac{1-2\ubar{\xi}}{N^{\nicefrac{3}{2}}}
  \sqrt{2\log(\nicefrac{4}{\eta_{\delta_N}})} \eqdef\delta_N \enspace.
\]
\end{lemma}
\begin{proof}
For fixed $P$, $a\in\cA$ and $t_a\in[0,1]$, we have
\[
  \big(\Phi_a(P;\xi)-t_a\big)^2=\big(\beta_a(P)-t_a\big)^2
  +2\big(\beta_a(P)-t_a\big)\gamma_a(P) \xi+\gamma_a(P)^2 \xi^2 \enspace.
\]
Therefore,
\begin{align*}
  \cC^a_\Xi(P;t_a)-\frac1N\sum_{k=1}^N\big(\Phi_a(P;\xi_k)-t_a\big)^2
  &=2\big(\beta_a(P)-t_a\big)\gamma_a(P)\Big\{\Exp_\Xi[\xi]-\frac{1}{N}
  \sum_{k=1}^N\xi_k\Big\}+\gamma_a(P)^2\Big\{\Exp_\Xi[\xi^2]-\frac1N
  \sum_{k=1}^N\xi_k^2\Big\} \enspace.
\end{align*}
We have $|\beta_a(P)-t_a|\leq 1$ and $|\gamma_a(P)|\leq 1$. Since the maps $\xi\mapsto\xi$
and $\xi\mapsto\xi^2$ are respectively $1$--Lipschitz and
$2(1-\ubar{\xi})$--Lipschitz on $[\ubar{\xi},1-\ubar{\xi}]$, applying
Lemma~\ref{lem:rates:stratif} to each of them with $\eta=\eta_{\delta_N}/2$, and a union bound, with probability greater than $1-\eta_{\delta_N}$ we obtain,
\[
  \Big|\cC^a_\Xi(P;t_a)-\frac1N\sum_{k=1}^N\big(\Phi_a(P;\xi_k)-t_a\big)^2
  \Big| \leq \big(2+2(1-\ubar{\xi})\big)\frac{1-2\ubar{\xi}}{N^{\nicefrac{3}{2}}}
  \sqrt{2\log(\nicefrac{4}{\eta_{\delta_N}})} \enspace.
\]
The proof is concluded.
\end{proof}

\subsubsection{Sample generalization rates}
\begin{lemma}
\label{lem:rates:Phi-gen} Conditionally on $(\cD^{n'},\Xi^N)$, for all
$\eta_{\delta_n}\in(0,1)$, with probability greater than $1-\eta_{\delta_n}$
over $\cD^n$, it holds that
\[
  \max_{a\in\cA,\ell\in[L],k\in[N]}\big|\hat\Phi^a_n(\hat P_\ell;\xi_k)
  -\Phi_a(\hat P_\ell;\xi_k)\big| \leq \frac{1}{\ubar{p}}\sqrt{\frac{\log(\nicefrac{2|\cA|LN}{\eta_{\delta_n}})}
{2n}}\eqdef\delta_n \enspace.
\]
Moreover, for all $\balpha\in\Delta_L$, $a\in\cA$ and
$k\in[N]$, it holds that $|\hat\Phi^a_n(\hat P^{\balpha};\xi_k)
-\Phi_a(\hat P^{\balpha};\xi_k)|\leq\delta_n$.

Similarly, for all
$\eta_{\delta_n(\cU)}\in(0,1)$, with probability greater than
$1-\eta_{\delta_n(\cU)}$, it holds that
\[
\sup_{\balpha\in\Delta_L}\big|\hat\cU_n(\hat P^{\balpha})
  -\cU(\hat P^{\balpha})\big| \leq \sqrt{\frac{\log(\nicefrac{2L}{\eta_{\delta_n(\cU)}})}
{2n}}\eqdef\delta_n(\cU) \enspace.
\]
\end{lemma}
\begin{proof}
Let us fix $a\in\cA$, $\ell\in[L]$ and $k\in[N]$, and the policy $\hat P_\ell$ is
fixed conditionally on $\cD^{n'}$. Let us define $\zeta_i\eqdef w_{\xi_k}(G_i) \oU_a(\hat P_\ell;\bZ_i)$, for
$\bZ_i\simiid\cD$. From Assumption~\ref{ass:groups} we have $\zeta_i\in[0,1/\ubar{p}]$, and it holds that $\Exp[\zeta_i]=\Phi_a(\hat P_\ell;\xi_k)$. Applying Hoeffding's inequality (Lemma~\ref{lem:hoeffding-ineq}), for any
$\epsilon>0$, it holds that
\[
  \Prob_{\cD^n}\Big(\Big|\frac{1}{n}\sum_{i=1}^n\zeta_i-\Phi_a(\hat P_\ell;\xi_k)
  \Big|\geq\epsilon \ \Big| \ \cD^{n'},\xi_k \Big) \leq2\exp\big(-2n\ubar{p}^{2}\epsilon^2\big) \enspace.
\]
Setting $\frac{\eta_{\delta_n}}{|\cA|LN}=2\exp(-2n\ubar{p}^{\,2}\epsilon^2)$,
obtaining $\epsilon=\frac{1}{\ubar{p}}\sqrt{\frac{\log(\nicefrac{2|\cA|LN}{\eta_{\delta_n}})}
{2n}}$, and applying a union bound over all triples
$(a,\ell,k)\in\cA\times[L]\times[N]$ concludes the proof of the first
inequality.

Moreover, by Lemma~\ref{lem:pol-lin}, the triangle inequality
and the fact that $\alpha_\ell\geq0$, $\sum_\ell\alpha_\ell=1$, we obtain
\[
  \big|\hat\Phi^a_n(\hat P^{\balpha};\xi_k)-\Phi_a(\hat P^{\balpha};\xi_k)\big|
  =\Big|\sum_{\ell\in[L]}\alpha_\ell\big(\hat\Phi^a_n(\hat P_\ell;\xi_k)
  -\Phi_a(\hat P_\ell;\xi_k)\big)\Big|
  \leq\max_{\ell\in[L]}\big|\hat\Phi^a_n(\hat P_\ell;\xi_k)
  -\Phi_a(\hat P_\ell;\xi_k)\big|\leq\delta_n \enspace.
\]

Finally,  for each fixed $\ell\in[L]$, the random variables
$\oU(\hat P_\ell;\bZ_i)$ are \iid conditionally on $\cD^{n'}$ and take values
in $[0,1]$. Hoeffding's inequality (Lemma~\ref{lem:hoeffding-ineq}) and a union bound over $\ell\in[L]$ yield
$\max_{\ell}|\hat\cU_n(\hat P_\ell)-\cU(\hat P_\ell)|\leq\delta_n(\cU)$ with
probability greater than $1-\eta_{\delta_n(\cU)}$. Linearity transfers the bound to the mixtures.

The proof is concluded.
\end{proof}

\subsection{Proof of Theorem~\ref{thm:rates:main}}
Finally, we are in position to prove the main theorem.
\subsubsection{Constraint rates: stability, unfairness, deployment}
Formal version and the proof of the first statement of Theorem~\ref{thm:rates:main}.
\begin{theorem}[Stability]
\label{thm:rates:stab}
Let $\hat\balpha\in\Delta_L$ be a vector \st $\hat\cC_{n,N}(P^{\hat\balpha};t_a)\leq\delta_a$ for all $a\in\cA$. For every $\eta_{\delta_n},\eta_{\delta_N}\in(0,1)$ and conditionally on $\cD^{n'}$, with probability greater than $1-\eta_{\delta_n}-\eta_{\delta_N}$ over $(\cD^n,\Xi^N)$, it holds that
\[
  \cC^a_\Xi(\hat P^{\hat\balpha};t_a) 
  \leq \big(\sqrt{\delta_a}+\delta_n\big)^2+\delta_N \enspace, \forall a\in\cA\enspace,
\]
where
\[
\delta_n \eqdef \frac{1}  {\ubar{p}}\sqrt{\frac{\log(\nicefrac{2|\cA|LN}{\eta_{\delta_n}})}{2n}} \quad \text{and} \quad \delta_N \eqdef \big(2+2(1-\ubar{\xi})\big)\frac{1-2\ubar{\xi}}{N^{\nicefrac{3}{2}}}
\sqrt{2\log(\nicefrac{4}{\eta_{\delta_N}})} \enspace.
\]
\end{theorem}
\begin{proof}
For a fixed $k\in[N]$, by the triangle inequality and Lemma~\ref{lem:rates:Phi-gen}, we have
\[
  \big|\Phi_a(\hat P^{\hat\balpha};\xi_k)-t_a\big|\leq
  \big|\hat\Phi^a_n(\hat P^{\hat\balpha};\xi_k)-t_a\big|
  +\big|\Phi_a(\hat P^{\hat\balpha};\xi_k)
  -\hat\Phi^a_n(\hat P^{\hat\balpha};\xi_k)\big|
  \leq\big|\hat\Phi^a_n(\hat P^{\hat\balpha};\xi_k)-t_a\big|+\delta_n \enspace.
\]
Squaring both sides and averaging over $k=1,\cdots,N$, we get
\[
  \frac{1}{N}\sum_{k=1}^N\big(\Phi_a(\hat P^{\hat\balpha};\xi_k)-t_a\big)^2
  \leq\hat\cC^a_{n,N}(\hat\balpha;t_a)+2\delta_n\,\frac{1}{N}\sum_{k=1}^N
  \big|\hat\Phi^a_n(\hat P^{\hat\balpha};\xi_k)-t_a\big|+\delta_n^2 \enspace.
\]
By the Cauchy--Schwarz inequality and the feasibility of $\hat\balpha$, i.e.,
$\hat\cC^a_{n,N}(\hat\balpha;t_a)\leq\delta_a$, we have
\[
  \frac{1}{N}\sum_{k=1}^N\big|\hat\Phi^a_n(\hat P^{\hat\balpha};\xi_k)-t_a\big|
  \leq\Big(\frac{1}{N}\sum_{k=1}^N\big(\hat\Phi^a_n(\hat P^{\hat\balpha};\xi_k)
  -t_a\big)^2\Big)^{1/2}=\hat\cC^a_{n,N}(\hat\balpha;t_a)^{1/2}
  \leq\sqrt{\delta_a} \enspace.
\]
Combining the two inequalities above, we get
\begin{equation}\label{eq:grid-bound}
  \frac{1}{N}\sum_{k=1}^N\big(\Phi_a(\hat P^{\hat\balpha};\xi_k)-t_a\big)^2
  \leq\delta_a+2\delta_n\sqrt{\delta_a}+\delta_n^2
  =\big(\sqrt{\delta_a}+\delta_n\big)^2 \enspace.
\end{equation}
Finally, applying Lemma~\ref{lem:rates:pert}
for $P=\hat P^{\hat\balpha}$, combined with \eqref{eq:grid-bound}, we get
\[
  \cC^a_\Xi(\hat P^{\hat\balpha};t_a)\leq\frac{1}{N}\sum_{k=1}^N
  \big(\Phi_a(\hat P^{\hat\balpha};\xi_k)-t_a\big)^2+\delta_N
  \leq\big(\sqrt{\delta_a}+\delta_n\big)^2+\delta_N \enspace.
\]
The proof is concluded.
\end{proof}

The second statement of Theorem~\ref{thm:rates:main} follows from Theorem~\ref{thm:rates:stab}.
\begin{corollary}[Unfairness] 
\label{cor:rate:unf} Under the conditions of Theorem~\ref{thm:rates:stab}, it holds that
\[
  \big|\gamma_a(\hat P^{\hat\balpha})\big|\leq \frac{\sqrt{\delta_a}+\delta_n+\sqrt{\delta_N}}{\sigma_\Xi} \enspace.
\]
\end{corollary}
\begin{proof}
From Theorem~\ref{thm:rates:stab} and from Definition~\ref{def:var} we have
\[
  \sigma_\Xi^2 \gamma_a(\hat P^{\hat\balpha})^2\leq
  \cC^a_\Xi(\hat P^{\hat\balpha};t_a)\leq\big(\sqrt{\delta_a}+\delta_n\big)^2
  +\delta_N \enspace.
\]
Taking the square root and using $\sqrt{u^2+w}\leq u+\sqrt w$ for
$u,w\geq0$, we obtain
\[
  \big|\gamma_a(\hat P^{\hat\balpha})\big|\leq
  \frac{\sqrt{(\sqrt{\delta_a}+\delta_n)^2+\delta_N}}{\sigma_\Xi}\leq
  \frac{\sqrt{\delta_a}+\delta_n+\sqrt{\delta_N}}{\sigma_\Xi} \enspace.
\]
The proof is concluded.
\end{proof}

Formal version and the proof of the third statement of Theorem~\ref{thm:rates:main}.
\begin{lemma}[Deployment]
\label{lem:rates:depl} Under the conditions of Theorem~\ref{thm:rates:stab}, simultaneously for all $a\in\cA$ and $\xi' \in [\ubar \xi, 1 - \ubar \xi]$, it holds that
\[
  \big|\Phi_a(\hat P^{\hat\balpha};\xi')-t_a\big|\leq \Big(1+\frac{|\xi'-\mu_\Xi|}{\sigma_\Xi}\Big)\big(\sqrt{\delta_a}+\delta_n
  +\sqrt{\delta_N}\big) \enspace.
\]
\end{lemma}
\begin{proof}
It holds that
\[
  \big(\Exp_\Xi[\Phi_a(\hat P^{\hat\balpha};\xi)]-t_a\big)^2\leq
  \cC^a_\Xi(\hat P^{\hat\balpha};t_a)\leq\big(\sqrt{\delta_a}+\delta_n\big)^2
  +\delta_N \enspace.
\]
Taking the square root, we get
$|\Exp_\Xi[\Phi_a(\hat P^{\hat\balpha};\xi)]-t_a|\leq\sqrt{\delta_a}+\delta_n
+\sqrt{\delta_N}$. For a fixed $\xi'\in[\ubar{\xi},1-\ubar{\xi}]$, we have
\[
  \Phi_a(\hat P^{\hat\balpha};\xi')-t_a
  =\underbrace{\Phi_a(\hat P^{\hat\balpha};\xi')
  -\Exp_\Xi[\Phi_a(\hat P^{\hat\balpha};\xi)]}_{\text{fluctuation term}}
  +\underbrace{\Exp_\Xi[\Phi_a(\hat P^{\hat\balpha};\xi)]-t_a}_{\text{bias term}} \enspace.
\]
Since $\Phi_a(P;\xi)=\beta_a(P)+\xi\gamma_a(P)$ and
$\Exp_\Xi[\Phi_a(P;\xi)]=\beta_a(P)+\mu_\Xi\gamma_a(P)$, the fluctuation term
equals $(\xi'-\mu_\Xi)\gamma_a(\hat P^{\hat\balpha})$. Applying the triangle
inequality and Corollary~\ref{cor:rate:unf}, we obtain
\[
  \big|\Phi_a(\hat P^{\hat\balpha};\xi')-t_a\big|\leq|\xi'-\mu_\Xi|\,
  \frac{\sqrt{\delta_a}+\delta_n+\sqrt{\delta_N}}{\sigma_\Xi}
  +\sqrt{\delta_a}+\delta_n+\sqrt{\delta_N}
  =\Big(1+\frac{|\xi'-\mu_\Xi|}{\sigma_\Xi}\Big)\big(\sqrt{\delta_a}+\delta_n
  +\sqrt{\delta_N}\big) \enspace.
\]
The proof is concluded.
\end{proof}

\subsubsection{Utility rates}
First, let us define $\epsilon_0\eqdef\min_{a\in\cA}\ens{\delta_a-\cC^a_\Xi(P^{\balpha^0};t_a)}$. Now, for $0\leq\epsilon\leq\epsilon_0$, let us define
\[
  \nu(\epsilon)\eqdef\max_{\balpha\in\Delta_L}\enscond{\cU(P^{\balpha})}
  {\cC^a_\Xi(P^{\balpha};t_a)\leq\delta_a-\epsilon,\ \forall a\in\cA} \enspace.
\]
Under Assumption~\ref{ass:feas}, $\balpha^0$ is feasible for every
$\epsilon\in[0,\epsilon_0]$, and in particular, $\nu(0)=\cU(P^{\balpha^\star})$.

Next, let us prove the following two lemmas that we will need later.
\begin{lemma}[Reverse feasibility]
\label{lem:feas-reverse}
Let $a\in\cA$ be such that
$\delta_n+\delta'_{n'}\leq\sqrt{\delta_a}$. If $\balpha\in\Delta_L$
satisfies
\(
\cC^a_\Xi(P^{\balpha};t_a)\leq\big(\sqrt{\delta_a}-\delta_n-\delta'_{n'}\big)^2-\delta_N ,
\)
then $\hat\cC^a_{n,N}(\balpha;t_a)\leq\delta_a$.
\end{lemma}
\begin{proof}
For a fixed $k\in[N]$, by the triangle inequality, Lemma~\ref{lem:rates:policy} and Lemma~\ref{lem:rates:Phi-gen}, it holds that
\begin{align*}
  \big|\hat\Phi^a_n(\hat P^{\balpha};\xi_k)-t_a\big|
  &\leq\big|\hat\Phi^a_n(\hat P^{\balpha};\xi_k)-\Phi_a(\hat P^{\balpha};\xi_k)
  \big|+\big|\Phi_a(\hat P^{\balpha};\xi_k)-\Phi_a(P^{\balpha};\xi_k)\big|
  +\big|\Phi_a(P^{\balpha};\xi_k)-t_a\big|\\
  &\leq\delta_n+\delta'_{n'}+\big|\Phi_a(P^{\balpha};\xi_k)-t_a\big|\enspace.
\end{align*}
Squaring and averaging over $k=1,\cdots,N$, we get
\[
  \hat\cC^a_{n,N}(\balpha;t_a)\leq\frac{1}{N}\sum_{k=1}^N
  \big(\Phi_a(P^{\balpha};\xi_k)-t_a\big)^2+2(\delta_n+\delta'_{n'})\,\frac{1}{N}\sum_{k=1}^N
  \big|\Phi_a(P^{\balpha};\xi_k)-t_a\big|+(\delta_n+\delta'_{n'})^2 \enspace.
\]
By Lemma~\ref{lem:rates:pert} and the assumption on $\balpha$, we have
\[
  \frac{1}{N}\sum_{k=1}^N\big(\Phi_a(P^{\balpha};\xi_k)-t_a\big)^2\leq
  \cC^a_\Xi(P^{\balpha};t_a)+\delta_N\leq\big(\sqrt{\delta_a}-\delta_n-\delta'_{n'}\big)^2 \enspace.
\]
From Cauchy--Schwarz inequality,
$\frac{1}{N}\sum_k|\Phi_a(P^{\balpha};\xi_k)-t_a|\leq\sqrt{\delta_a}-\delta_n-\delta'_{n'}$.

Finally, combining the above, we get
\[
  \hat\cC^a_{n,N}(\balpha;t_a)\leq\big(\sqrt{\delta_a}-\delta_n-\delta'_{n'}\big)^2
  +2(\delta_n + \delta'_{n'})\big(\sqrt{\delta_a}-\delta_n-\delta'_{n'}\big)+(\delta_n+\delta'_{n'})^2=
  \delta_a \enspace.
\]
The proof is concluded.
\end{proof}

\begin{lemma}[Value continuity]
\label{lem:val-cont}
Under Assumption~\ref{ass:feas}, for every $0\leq\epsilon\leq\epsilon_0$, it
holds that
\[
  0\leq\nu(0)-\nu(\epsilon)\leq\epsilon\,C_0 ,\qquad\text{where}\qquad
  C_0\eqdef\frac{\nu(0)-\cU(P^{\balpha^0})}{\epsilon_0} \enspace.
\]
\end{lemma}
\begin{proof}
Let us fix $\epsilon\in[0,\epsilon_0]$ and set
\[
  \lambda_\epsilon\eqdef\frac{\epsilon}{\epsilon_0}\in[0,1] ,\qquad
  \balpha_\epsilon\eqdef(1-\lambda_\epsilon)\balpha^\star+\lambda_\epsilon
  \balpha^0 \enspace.
\]
The simplex $\Delta_L$ is convex, hence $\balpha_\epsilon\in\Delta_L$.
Moreover, $\balpha\mapsto\Phi_a(P^{\balpha};\xi)$ is linear for every fixed
$\xi$ (Lemma~\ref{lem:pol-lin}), hence $\balpha\mapsto(\Phi_a(P^{\balpha};
\xi)-t_a)^2$ is convex, and taking the expectation over $\xi\sim\Xi$
preserves convexity. Therefore, for every $a\in\cA$, it holds that
\[
  \cC^a_\Xi(P^{\balpha_\epsilon};t_a)\leq(1-\lambda_\epsilon)\,
  \cC^a_\Xi(P^{\balpha^\star};t_a)+\lambda_\epsilon\,
  \cC^a_\Xi(P^{\balpha^0};t_a)\leq(1-\lambda_\epsilon)\,\delta_a
  +\lambda_\epsilon\,(\delta_a-\epsilon_0)=\delta_a-\epsilon \enspace.
\]
Hence, $\balpha_\epsilon$ is feasible at level $\epsilon$, and thus
$\nu(\epsilon)\geq\cU(P^{\balpha_\epsilon})$. Using the linearity of
$\balpha\mapsto\cU(P^{\balpha})$ (Lemma~\ref{lem:pol-lin}), we obtain
\[
  \cU(P^{\balpha_\epsilon})=(1-\lambda_\epsilon)\,\cU(P^{\balpha^\star})
  +\lambda_\epsilon\,\cU(P^{\balpha^0})=\nu(0)-\lambda_\epsilon
  \big(\nu(0)-\cU(P^{\balpha^0})\big)=\nu(0)-\epsilon\,C_0 \enspace.
\]
Finally, combining the above yields $\nu(0)-\nu(\epsilon)\leq\epsilon C_0$. We
conclude the proof by noticing that $\nu(\epsilon)\leq\nu(0)$ due to the
monotonicity of the feasible sets.
\end{proof}

Finally, the formal version and the proof of the fourth statement of Theorem~\ref{thm:rates:main}.
\begin{theorem}[Utility]
\label{thm:rates:depl}
Let Assumptions~\ref{ass:v-bd},~\ref{ass:groups},~\ref{ass:feas},~\ref{ass:pol-exp} hold. Let
\[
  \epsilon_0\eqdef\min_{a\in\cA}\ens{\delta_a-\cC^a_\Xi(P^{\balpha^0};t_a)}\enspace,
  \qquad
  \delta'_{n'}\eqdef\frac{\delta_{n'}}{\ubar{p}\,\ubar v}\enspace,
  \qquad
  \kappa\eqdef\max_{a\in\cA}2\sqrt{\delta_a}\,\big(\delta_n+\delta'_{n'}\big)
  +\delta_N\enspace,
\]
where $\delta_n$ and $\delta_N$ are as in Theorem~\ref{thm:rates:stab}, and
suppose that $n,n',N$ are large enough so that $\kappa\leq\epsilon_0$. Then,
for all $\eta_{\delta_{n'}},\eta_{\delta_n},\eta_{\delta_N},
\eta_{\delta_n(\cU)}\in(0,1)$, with probability greater than
$1-\eta_{\delta_{n'}}-\eta_{\delta_n}-\eta_{\delta_N}-\eta_{\delta_n(\cU)}$
over $(\cD^{n'},\cD^n,\Xi^N)$, problem \eqref{eq:prob:emp} is
feasible, and every solution $\hat\balpha$ of \eqref{eq:prob:emp} satisfies
\[
  \cU(P^{\balpha^\star})-\cU(\hat P^{\hat\balpha})\leq C_0\kappa
  +2\delta_n(\cU)+\frac{\delta_{n'}}{\ubar v} ,
\]
where
\[
  C_0\eqdef\frac{\cU(P^{\balpha^\star})-\cU(P^{\balpha^0})}{\epsilon_0}
  \in\Big[0,\frac{1}{\epsilon_0}\Big]
  \qquad\text{and}\qquad
  \delta_n(\cU)=\sqrt{\frac{\log(2L/\eta_{\delta_n(\cU)})}{2n}} .
\]
\end{theorem}
\begin{proof}
Let $\tilde\balpha$ be a maximizer in the
definition of $\nu(\kappa)$, which exists since $\kappa\leq\epsilon_0$, so that
$\cU(P^{\tilde\balpha})=\nu(\kappa)$ and, for every $a\in\cA$,
\[
  \cC^a_\Xi(P^{\tilde\balpha};t_a)\leq\delta_a-\kappa\leq\delta_a
  -2\sqrt{\delta_a}\,(\delta_n+\delta'_{n'})-\delta_N\leq\big(\sqrt{\delta_a}-(\delta_n+\delta'_{n'})\big)^2-\delta_N \enspace,
\]
where the second inequality follows from the definition of $\kappa$. Moreover,
since $\cC^a_\Xi(P^{\balpha^0};t_a)\geq0$, we have $\epsilon_0\leq\delta_a$,
hence $2\sqrt{\delta_a}\,(\delta_n+\delta'_{n'})\leq\kappa\leq\epsilon_0\leq\delta_a$,
that is $\delta_n+\delta'_{n'}\leq\nicefrac{\sqrt{\delta_a}}{2}$. Therefore, applying Lemma~\ref{lem:feas-reverse}, we get $\hat\cC^a_{n,N}(\tilde\balpha;t_a)\leq\delta_a$ for all $a\in\cA$. Thus, $\tilde\balpha$ is feasible for~\eqref{eq:prob:emp}, which is thus non-empty and
admits a solution $\hat\balpha$.

Since $\hat\balpha$ is a maximizer of
\eqref{eq:prob:emp} and $\tilde\balpha$ is feasible for \eqref{eq:prob:emp}, we have
$\hat\cU_n(\hat P^{\hat\balpha})\geq\hat\cU_n(\hat P^{\tilde\balpha})$.
Combining this with Lemma~\ref{lem:rates:Phi-gen}, we get
\[
  \cU(\hat P^{\hat\balpha})\geq\hat\cU_n(\hat P^{\hat\balpha})-\delta_n(\cU)
  \geq\hat\cU_n(\hat P^{\tilde\balpha})-\delta_n(\cU)
  \geq\cU(\hat P^{\tilde\balpha})-2\delta_n(\cU) \enspace.
\]
Next, applying Lemma~\ref{lem:rates:policy}, we  obtain
\[
  \cU(\hat P^{\tilde\balpha})\geq\cU(P^{\tilde\balpha})-\frac{\delta_{n'}}{\ubar v}
  =\nu(\kappa)-\frac{\delta_{n'}}{\ubar v} \enspace.
\]
Finally, applying Lemma~\ref{lem:val-cont} with $\epsilon=\kappa\leq\epsilon_0$, we get
$\nu(\kappa)\geq\nu(0)-\kappa C_0=\cU(P^{\balpha^\star})-\kappa C_0$. Combining everything above, we  obtain
\[
  \cU(P^{\balpha^\star})-\cU(\hat P^{\hat\balpha})\leq C_0\kappa
  +2\delta_n(\cU)+\frac{\delta_{n'}}{\ubar v} ,
\]
where $0\leq C_0\leq \nicefrac{1}{\epsilon_0}$, since $\balpha^0$ is feasible for
\eqref{eq:prob:pop} and $\cU(P^{\balpha})\in[0,1]$ for every $\balpha\in\Delta_L$.

The proof is concluded.
\end{proof}
\section{Experiments: additional details}
\label{app:experiments}

\subsection{Implementation Details}
\label{app:impl}

\paragraph{Utility.} Policies rank the candidates of a user by decreasing score, and the uniform policy places every candidate at every position with
equal probability. The position weights are $v_j \propto 1/\log_2(j+1)$ for $j \le K = 20$ and $v_j = 0$ beyond, so that the utility $U(P,\bz)$ is the
NDCG@20 with binary relevance. The utility of category $a$ is $U^a(P,\bz) = \sum_{i \in \mathcal I_a} \rho_i(P;\bx)\, y_i / V(\bz)$, with the same normalisation $V(\bz)$ as $U$, so that $\sum_a U^a = U$.

\paragraph{Rankers.} The history of a user is encoded with PureSVD: $Q$ holds the top $d$ right singular vectors of the binary history matrix of the
training users, and $e_u = h_u Q$ for any user with history $h_u$. The features of a pair $(u,i)$ are $\langle e_u, q_i\rangle$, $e_u \odot q_i$,
$\log(1 + n_i)$, where $n_i$ is the number of training histories containing $i$, and the length of the history of $u$. The sensitive attribute is never a
feature. Rankers are LambdaMART models trained with LightGBM (150 trees, learning rate $0.1$, $31$ leaves, at least $40$ samples per leaf, row and
column subsampling $0.9$ and $0.8$), on all relevant candidates of the training users and $n_{\rm neg}$ non-relevant candidates per user drawn at
random. The group-reweighted rankers use the weights $w_\xi(g) = \xi/p_1$ for $g = 1$ and $(1-\xi)/(1-p_1)$ for $g = 0$, with $p_1$ the share of women among training users.

\paragraph{Problem~\eqref{eq:prob:emp}} $\Xi$ is uniform on $[\underline\xi, 1-\underline\xi]$ with $\underline\xi = 0.1$, so that
$\mu_\Xi = 1/2$ and $\sigma^2_\Xi = (1-2\underline\xi)^2/12$. We use $N = 20$ values of $\xi$, one drawn uniformly in each of $N$ intervals of equal width covering the support. The budgets are $\delta_a = \sigma^2_\Xi (\lambda\,\varepsilon_a)^2$, where $\varepsilon_a$ is twice the standard error of the gap $\widehat\gamma_a(P_0)$ on the estimation users. Since the perturbation variance of category $a$ is $\sigma^2_\Xi\,\gamma_a^2$, the constraint without target amounts to $|\gamma_a| \le \lambda\,\varepsilon_a$: with $\lambda = 1$, the tolerated gap is the estimation noise of the gap of $P_0$. The aggregate constraint of UP uses the same rule with the aggregate gap. The targets are $t_a = \kappa\,\widehat\Phi_a(P_0;\mu_\Xi)$, computed on the estimation users. Problem~\eqref{eq:prob:emp} is solved with CVXPY and the Clarabel solver.

\subsection{Synthetic Data}
\label{app:toy}

\paragraph{Ground truth.} Each of the $8000$ users belongs to group $G = 1$ (women) with probability $0.4$. The $300$ offers are split into three
categories of equal size. Users have a latent taste $u \sim \mathcal N(0, I_8)$ and an affinity $a_{u,c} \sim \mathcal N(0,1)$ for each category $c$, both independent of $G$, and offers have an embedding $q_i \sim \mathcal N(0, I_8/8)$. The true interest of user $u$ in offer $i$ of
category $c(i)$ is $y^\star_{ui} \sim \mathrm{Bernoulli}\big(\mathrm{sigmoid}(\langle u, q_i\rangle + 1.5\, a_{u,c(i)} - 3)\big)$. Since the interests do not depend on $G$, an oracle ranking by $y^\star$ has no gap in any category.

\paragraph{Historical bias.} The recorded interactions keep each true interest, except those of women in managerial offers and of men in administrative offers, which are kept with probability $1-\beta$ only. The two biases have the same strength, so that the aggregate gap remains small
while the gaps of the two categories have opposite signs.

\paragraph{Histories, candidates and splits.} For each user, the catalogue is split at random, independently of $G$, into a history pool ($60\%$ of the offers) and a candidate pool. The history is the set of recorded interactions in the history pool, and the candidates are the offers of the candidate pool. Users are split at random into training ($50\%$), estimation ($25\%$) and test ($25\%$) users. The relevant candidates are the recorded interactions for training users, and the true interests for estimation and test users. Users with fewer than $3$ history items or no relevant candidate are removed.

\paragraph{Settings.} The dictionary is described in Section~\ref{exp:toy}, with $d = 16$ and $n_{\rm neg} = 50$. We use $\beta \in \{0, 0.25, 0.5, 0.75\}$, $\lambda \in \{1, 2, 3, 5, 8\}$ and $\kappa \in \{1, 0.9, \dots, 0.4\}$, and repeat each setting with $3$ seeds, which change the data, the splits and the rankers.

\paragraph{Deployment on new populations.} \label{app:deploy} To illustrate
the uniform deployment guarantee (Theorem~1.3), we deploy the mixtures
selected for $\beta = 0.5$ and $\lambda = 5$ on new populations of $4000$
users, drawn with the same catalogue and a share of women
$p' \in \{0.1, 0.3, 0.5, 0.7, 0.9\}$, and labelled with their true interests.
Figure~\ref{fig:deploy} shows that the utility of $P_0$ over all offers does
not depend on $p'$, whereas the utility of managerial and administrative
offers moves by about $20\%$ in opposite directions. CUP reduces this
dependence as $\kappa$ decreases, at the cost of utility shown in
Figure~\ref{fig:toy}(d). For CUP, the distance between the utility of each
category and its target stays below
$\lambda\varepsilon_a(\sigma_\Xi + |p' - \mu_\Xi|)$, the leading term of the
bound of Theorem~1.3, in $85$ of the $90$ cases (three categories, five
shares, three seeds and two values of $\kappa$).

\begin{figure}[t]
\centering
\includegraphics[width=\linewidth]{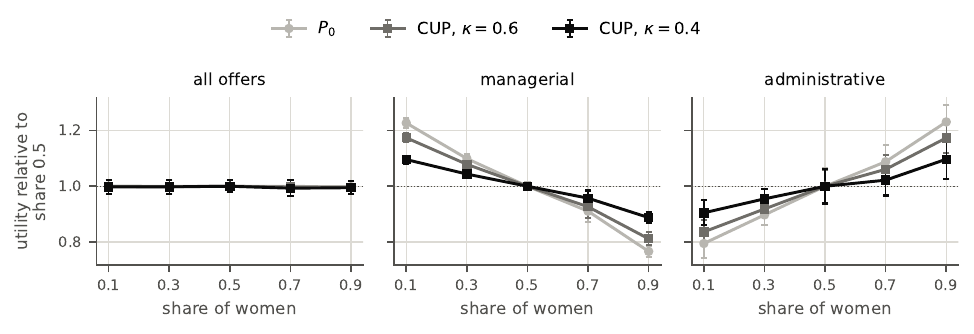}
\caption{Synthetic data, deployment on new users ($\beta = 0.5$,
$\lambda = 5$). Utility of the mixtures selected on the estimation users
(share of women $0.4$), as a function of the share of women among new users,
relative to its mean value at $0.5$. Means and standard deviations over 3
seeds.}
\label{fig:deploy}
\end{figure}

\subsection{MovieLens-1M}
\label{app:ml1m}

\paragraph{Data.} Every rating counts as an interaction. We keep the movies that belong to exactly one of the six genres, and the users with at least
$10$ history items and at least one relevant item that is not in their history. Table~\ref{tab:ml1m-stats} gives the resulting statistics. The splits of users change with the seed, and so do the rankers.

\begin{table}[h]
\caption{MovieLens-1M after preprocessing: $5018$ users, of which $28\%$ women, and $1280$ movies. On average, a user has $61$ history items (median $39$), $16$ relevant items (median $10$) and $1219$ candidates. Share of relevant items: share of the relevant items of each group that belong to the genre.}
\label{tab:ml1m-stats}
\centering
\begin{tabular}{lrrcc}
\toprule
genre & movies & relevant items & \multicolumn{2}{c}{share of relevant items} \\
 & & & women & men \\
\midrule
Sci-Fi & 138 & 13953 & 0.14 & 0.19 \\
Adventure & 108 & 8155 & 0.08 & 0.11 \\
Crime & 182 & 11615 & 0.12 & 0.15 \\
Romance & 424 & 24360 & 0.40 & 0.28 \\
Children's & 159 & 10307 & 0.15 & 0.12 \\
Horror & 269 & 11134 & 0.11 & 0.15 \\
\midrule
total & 1280 & 79524 & & \\
\bottomrule
\end{tabular}
\end{table}

\paragraph{Dictionary.} The $L = 19$ policies are listed in Table~\ref{tab:ml1m-dict}, with $d = 64$ and $n_{\rm neg} = 100$. A cell is a
pair (group, genre), and the under-represented group of a genre is the group for which this genre has the smallest share of its relevant items, among
training users. The cell-based rankers only reweight the relevant candidates, so that $G$ enters through the weights and is never a feature. We use $\kappa \in \{1, 0.95, 0.9, 0.85, 0.8\}$ and $\lambda \in \{1, 1.5, 2, 3, 5, 8\}$.

\begin{table}[h]
\caption{Dictionary on MovieLens-1M. All policies but the first three
are LambdaMART models trained on the same rows.}
\label{tab:ml1m-dict}
\centering
\begin{tabular}{lp{0.62\linewidth}}
\toprule
policy & description \\
\midrule
uniform & every candidate at every position with equal probability \\
popularity & $\log(1+n_i)$, the same ranking for every user \\
PureSVD & $\langle e_u, q_i\rangle$ \\
$P_0$ & unweighted \\
$\xi$-reweighted (5) & weights $w_\xi(g)$, $\xi \in \{0.1, 0.3, 0.5, 0.7, 0.9\}$ \\
one group (2) & trained on the users of one group, deployed on all users \\
cell-balanced & each relevant item weighted by the inverse number of relevant
items of its cell \\
under-represented (1) & relevant items of the under-represented group of every
genre weighted by $5$ \\
under-represented, genre $a$ (6) & the same, for genre $a$ only \\
\bottomrule
\end{tabular}
\end{table}

\paragraph{ITFR and MF-BPR.} \label{app:itfr} We use the code of \citet{wang2024}, with minimal changes that leave the training algorithm untouched (device handling, and calls deprecated in recent versions of NumPy and PyTorch). The user groups are the two genders and the item groups the six genres. The training interactions are the histories of all users and $90\%$ of the relevant items of the training users; the remaining $10\%$ form the validation set. The relevant items of estimation and test users are never used. We keep the hyperparameters of the authors for MovieLens-1M: embeddings of dimension $64$, Adam with learning rate $10^{-3}$, no $\ell_2$ regularisation, batches of $1024$, one negative per positive, $\eta = 0.05$, $\gamma = 0.5$, $\rho = 0.02$ and $\tau = 0.2$. The model of the epoch with the best validation NDCG@20 within $100$ epochs is kept. Candidates are scored by the cosine between user and item embeddings, as in ITFR, and ranked deterministically like the rankers of the dictionary. MF-BPR is trained with the same code and hyperparameters, without the sharpness-aware perturbation and with group weights equal to the share of each group in the batch, which gives the BPR loss; it keeps the normalisation of the scores. 

\paragraph{CatFAIR.} \label{app:catfair} We re-implement in PyTorch the MF model of the code of \citet{kheya2025}, on the same training interactions and validation set as ITFR. The model has $40$ factors, user and item biases and a global bias, initialised as in their code, and is trained with the BPR loss, one negative per positive drawn uniformly, Adam with learning rate $10^{-3}$, weight decay $0.02$ and batches of $1024$. For the users of each batch, a differentiable top-$50$ over all items (hard in the forward pass, $\mathrm{sigmoid}((s - s_{50})/0.01)$ in the backward pass) gives the share of each genre in their list; the regularizer is the sum over genres of the absolute difference between the mean shares of women and men (eq.~10 of their paper), passed through $\mathrm{sigmoid}(0.1\,(g - 0.5))$
and weighted by $\alpha$ against the BPR loss, as in their code. Training stops when the validation NDCG@20 has not improved by $0.0005$ for $10$ epochs, within $100$ epochs. We use their grid $\alpha \in \{0, 0.1, \dots, 0.6\}$, where $\alpha = 0.6$, and $\alpha = 0$ gives the backbone. Candidates are scored by the inner product plus biases. The variant of the released code, which divides the share of
each genre by its share in the catalogue, gives comparable results.

\subsection{Additional Results on MovieLens-1M}
\label{app:ml1m-more}

\paragraph{Feasibility.} Table~\ref{tab:feas} gives the number of splits on which $(\widehat{\mathcal P})$ is feasible for CUP. UP is feasible on every
split except for $\kappa = 1$ and $\lambda = 1$ ($8/10$), and the versions without target are feasible on every split. With $\lambda = 2$, the targets are met on every split for $\kappa \le 0.9$.

\begin{table}[h]
\caption{Number of splits (out of 10) on which CUP is feasible.}
\label{tab:feas}
\centering
\begin{tabular}{lccccc}
\toprule
 & \multicolumn{5}{c}{$\kappa$} \\
$\lambda$ & $1$ & $0.95$ & $0.9$ & $0.85$ & $0.8$ \\
\midrule
$1$   & 0  & 0  & 0  & 0  & 1  \\
$1.5$ & 1  & 3  & 5  & 7  & 7  \\
$2$   & 6  & 9  & 10 & 10 & 10 \\
$\ge 3$ & 10 & 10 & 10 & 10 & 10 \\
\bottomrule
\end{tabular}
\end{table}

\paragraph{Computation time.} Table~\ref{tab:time} details the times of
Section~\ref{sec:experiments}. Times are measured on a MacBook Pro with an Apple M2 chip and 16\,GB of memory, on CPU only. The dictionary times are averaged over $3$ splits run with an empty cache.
Training times are medians, validation included: over $20$
trainings for ITFR (about $23$\,s per epoch), and over the trainings of
splits $1$ to $9$ for CatFAIR ($9$ for $\alpha = 0$, $54$ for $\alpha > 0$,
about $3$\,s per epoch).
\begin{table}[h]
\caption{Computation time on MovieLens-1M, per split.}
\label{tab:time}
\centering
\begin{tabular}{lr}
\toprule
step & time \\
\midrule
\multicolumn{2}{l}{\emph{UP and CUP}} \\
PureSVD and features & $0.2$\,s \\
fitting the $16$ LightGBM rankers & $50$\,s \\
scoring the candidates of estimation and test users & $81$\,s \\
utilities of the $19$ policies & $1.6$\,s \\
one SOCP, UP / CUP & $2.5$ / $7.4$\,ms \\
the $72$ SOCPs of the grid & $0.4$\,s \\
\midrule
\multicolumn{2}{l}{\emph{Baselines, one training}} \\
ITFR ($100$ epochs) & $38$\,min \\
CatFAIR, $\alpha = 0$ & $26$\,s \\
CatFAIR, $\alpha > 0$ & $2.3$\,min \\
\bottomrule
\end{tabular}
\end{table}

\end{document}